\documentclass[10pt]{article} 

\usepackage[accepted]{rlj} 

\usepackage{amssymb}            
\usepackage{mathtools}          
\usepackage{mathrsfs}           
\usepackage{graphicx}           
\usepackage{subcaption}         
\usepackage[space]{grffile}     
\usepackage{url}                
\usepackage{lipsum}             

\usepackage{amsmath,amsfonts,bm,amsthm}

\def\1{\bm{1}}

\DeclareMathAlphabet{\mathsfit}{\encodingdefault}{\sfdefault}{m}{sl}
\SetMathAlphabet{\mathsfit}{bold}{\encodingdefault}{\sfdefault}{bx}{n}

\newcommand{\E}{\mathbb{E}}

\newcommand{\R}{\mathbb{R}}

\newcommand{\KL}{D_{\mathrm{KL}}}

\newcommand{\pikp}{{\pi_{k+1}}}

\newcommand{\pik}{{\pi_{k}}}

\newcommand{\Qtk}{{Q_{\tau}^{k}}}

\newcommand{\Qtkp}{{Q_{\tau}^{k+1}}}
\newcommand{\qkp}{{q_{k+1}}}

\newcommand{\Qtp}{{Q_{\tau}^{\pi}}}

\newcommand{\Vtp}{{V_{\tau}^{\pi}}}

\newcommand{\Ttkp}{{T_{\tau}^{k+1}}}
\newcommand{\Tts}{{T_{\tau}^{\star}}}
\newcommand{\Ttk}{{T_{\tau}^{k}}}
\newcommand{\Qts}{{Q^{\star}_\tau}}

\newcommand{\Vts}{{V^{\star}_\tau}}
\newcommand{\xik}{{\xi_k}}
\newcommand{\xikp}{{\xi_{k+1}}}

\newcommand{\Rx}{{R_\mathrm{x}}}
\newcommand{\pis}{{\pi^{\star}}}

\newcommand{\bR}{\bar{R}}

\newcommand{\epik}{{\epsilon_{\Delta_k}}}

\newcommand{\epiim}{{\epsilon_{\Delta_{i-1}}}}

\newcommand{\dqk}{{\Delta_q^{k}}}

\newcommand{\norm}[1]{\left\lVert#1\right\rVert}
\newcommand{\ninf}[1]{\norm{#1}_\infty}

\newcommand{\err}{{T}}

\usepackage{hyperref}
\usepackage{url}
\usepackage{graphicx}
\usepackage{wrapfig}
\usepackage{algorithm}
\usepackage{algorithmic}
\usepackage{subcaption}
\usepackage{multirow}
\usepackage{booktabs}       

\newcommand{\refp}[1]{\ref{#1}}

\newtheorem{prop}{Proposition}[section]
\newtheorem{thm}{Theorem}[section]
\newtheorem{lem}[thm]{Lemma}

\definecolor{darkgreen}{RGB}{0,180,30}
\newcommand{\dg}[1]{}

\title{StaQ: a Finite Memory Approach to Discrete Action Policy Mirror Descent}

\setrunningtitle{StaQ: a Finite Memory Approach to Discrete Action Policy Mirror Descent}

\author{Alex Davey\textsuperscript{1,2}, Alena Shilova\textsuperscript{1}, Brahim Driss\textsuperscript{2}, Riad Akrour\textsuperscript{2}}

\emails{\{alex.davey,alena.shilova,brahim.driss,riad.akrour\}@inria.fr}

\affiliations{
$^{1}$\textbf{Inria TAU, LISN, Université Paris-Saclay, Orsay, France}\\
$^{2}$\textbf{Univ. Lille, Inria, CNRS, Centrale Lille, UMR 9189 -- CRIStAL, Lille, France}
}

\contribution{
    We extend previous Policy Mirror Descent (PMD) analyses in two ways: i) we show that the averaging effect of policy evaluation errors can be obtained with a $\KL$ penalization of only the policy update, and ii) we extend the analysis to Policy Iteration (PI).  
    }
    {
    \cite{Kozuno19,Vieillard20} analyzed a PMD Value Iteration (VI) algorithm, showing the averaging effect of the $\KL$ regularization, but needed a $\KL$ regularized policy update \textbf{and} evaluation; PI was also left as an open problem. \cite{Cen,PMD} analyzed a PI scheme in the same setting as ours---entropy and $\KL$ regularized policy update and entropy regularized policy evaluation---but in their analysis, the effect of policy evaluation errors is minimized when the $\KL$ weight is set to zero, contradicting results of \cite{Kozuno19,Vieillard20}.
    }

\contribution{
    We propose and analyze StaQ, a PMD variant for discrete action spaces that keeps in memory the last $M$ Q-functions, and show that it converges for $M$ large enough, with no errors due to policy update, and that the averaging of policy evaluation errors is similar to that of exact PMD up to some extra terms that decay exponentially fast w.r.t. $M$.
    }
    {
    Exact PMD policy is an infinitely increasing sum over previous Q-functions. Existing approaches adopt an actor-critic view to approximate this policy, introducing policy update errors. We propose an alternative update that truncates the infinite sum to $M$.
    }

\contribution{
    We perform an efficient batched implementation of StaQ leveraging GPU parallelization and show that increasing $M$ has beneficial effects on performance, with diminishing returns, to the point that StaQ with a high enough $M$ can become indistinguishable from exact PMD.
    }
    {
    Our efficient implementation allows us to evaluate exact PMD (for up to 5 million time-steps) on medium sized MinAtar \citep{Minatar} environments, providing a test bed for comparing approximate and exact PMD algorithms.
    }


\summary{In Reinforcement Learning (RL), regularization with a Kullback-Leibler divergence that penalizes large deviations between successive policies has emerged as a popular tool both in theory and practice. This family of algorithms, often referred to as Policy Mirror Descent~(PMD), has the property of averaging out policy evaluation errors which are bound to occur when using function approximators. However, exact PMD has remained a mostly theoretical framework, as its closed-form solution involves the sum of all past Q-functions which is generally intractable. A common practical approximation of PMD is to follow the natural policy gradient or use actor-critic approaches, but this potentially introduces errors in the policy update. In this paper, we propose and analyze PMD-like algorithms for discrete action spaces that only keep the last $M$ Q-functions in memory. We show theoretically that for a finite and large enough $M$, an RL algorithm can be derived that introduces no errors from the policy update, yet keeps the desirable PMD property of averaging out policy evaluation errors. Using an efficient GPU implementation, we then show empirically on medium-scale RL benchmarks such as MinAtar that increasing $M$ improves performance up to a certain threshold after which the performance becomes close to that of exact PMD, reinforcing the theoretical findings that using an infinite sum might be unnecessary and that keeping in memory the last M Q-functions is a practical and theoretically grounded implementation of PMD.
}

\begin{document}

\makeCover  
\maketitle  

\begin{abstract}
In Reinforcement Learning (RL), regularization with a Kullback-Leibler divergence that penalizes large deviations between successive policies has emerged as a popular tool both in theory and practice. This family of algorithms, often referred to as Policy Mirror Descent (PMD), has the property of averaging out policy evaluation errors which are bound to occur when using function approximators. However, exact PMD has remained a mostly theoretical framework, as its closed-form solution involves the sum of all past Q-functions which is generally intractable. A common practical approximation of PMD is to follow the natural policy gradient or use actor-critic approaches, but this potentially introduces errors in the policy update. In this paper, we propose and analyze PMD-like algorithms for discrete action spaces that only keep the last $M$ Q-functions in memory. We show theoretically that for a finite and large enough $M$, an RL algorithm can be derived that introduces no errors from the policy update, yet keeps the desirable PMD property of averaging out policy evaluation errors. Using an efficient GPU implementation, we then show empirically on medium-scale RL benchmarks such as MinAtar that increasing $M$ improves performance up to a certain threshold after which the performance becomes close to that of exact PMD, reinforcing the theoretical findings that using an infinite sum might be unnecessary and that keeping in memory the last M Q-functions is a practical and theoretically grounded implementation of PMD.
\end{abstract}

\section{Introduction}
\label{sec:intro}
Deep RL has seen rapid development in the past decade, achieving super-human results on several decision making tasks~\citep{dqn, Go, GT}. However, the use of neural networks as function approximators exacerbates many challenges of RL, such as the brittleness to hyperparameters~\citep{henderson18} and the poor alignment between empirical behavior and theoretical understandings~\citep{Ilyas2020A, Discor, deadlytriad}. To address these issues, many deep RL algorithms consider adding regularization terms, one of which being to penalize the Kullback-Leibler divergence (labeled $\KL$ in the following) between successive policies~ \citep{TRPO,acktr,acer,MDQN}. This family of algorithms is often called Policy Mirror Descent~(PMD, \citet{Politex, improvedpolitex, PMD}) for its connection---made more explicit in Sec.~\ref{sec:pmd}---to the first-order convex optimization method Mirror Descent~\citep{Nemirovsky1983,Beck2003}. One known property of PMD algorithms, at least in the context of value iteration, is that it averages out policy evaluation errors \citep{Vieillard20}. This property can be intuited from the nature of the policy $\pi_k$ at iteration $k$, which is a Boltzmann distribution and its (unnormalized) log-probabilities are a weighted average of past Q-function estimates $\{q_{i}\}_{i=0}^{k-1}$
\begin{align}
    \pi_k\propto \exp\left(\alpha \sum_{i=0}^{k-1} \beta^i q_{k-i}\right), \label{eq:pol}
\end{align}
for temperature $\alpha > 0$ and weight $\beta\in(0,1]$. In contrast, unregularized RL would pick actions according to $\arg\max_a q_{k-1}(\cdot, a)$, which would be very sensitive to the potentially large policy evaluation errors that can occur in deep RL \citep{Ilyas2020A}. 

While averaging Q-function estimates might cancel out their errors, implementing Eq.~\ref{eq:pol} exactly can become intractable due to the infinite nature of the sum. This is especially true when a non-linear function approximator is used for $\{q_{i}\}_{i=0}^{k-1}$, precluding the existence of a closed-form expression for the policy. As such, prior work considered several types of approximation to PMD, such as following the natural gradient \citep{NPG, petersNPG, TRPO} or performing a few gradient steps over the regularized policy update loss \citep{PPO,tomar2022mirror}. Instead, we study both theoretically and empirically an alternative update---that we call StaQ, that is a weighted average of at most the last $M$ Q-function estimates
\begin{align}
    \pi_k\propto \exp\left(\frac{\alpha}{1-\beta^M} \sum_{i=0}^{M-1} \beta^i q_{k-i}\right).\label{eq:intro:staq}
\end{align}
From a theoretical point of view, we show that for $M$ large enough, such a truncation does not introduce a policy update error, i.e. that in the absence of policy evaluation errors, StaQ will converge exactly to the optimal policy. Moreover, StaQ has a similar averaging of error property as PMD, up to some extra terms that decay exponentially fast w.r.t. $M$; suggesting that as we increase $M$, we can quickly recover the behavior of exact PMD. From a practical point of view, the study of such an algorithm is timely in the age of batched GPU computations: indeed we show that for medium-sized problems such as the image-based MinAtar \citep{Minatar} tasks, increasing $M$ has little to no impact on the run-time, making Eq.~\ref{eq:intro:staq} a practical alternative policy update to actor-critic approaches. This is especially true since the policy update in the discrete action case, to which the scope of this paper is limited to, is optimization free.  
\section{Related Work}
\label{sec:related}
\textbf{Regularization in RL.} Regularization has seen widespread usage in RL. It was used with~(natural) policy gradient~\citep{NPG, TRPO, yuan2022linear}, policy search~\citep{psearch}, policy iteration~\citep{Politex, PMD} and value iteration methods~\citep{SoftQ, MDQN}. Common choices of regularizers include minimizing the $\KL$ between the current and previous policy~\citep{DPP, TRPO} or encouraging high Shannon entropy~\citep{SoftQ, SAC_app}, but other regularizers exist~\citep{Tsallis, ampo}. We refer the reader to~\citet{Neu,Geist19} for a broader categorization of entropy regularizers and their relation to existing deep RL methods. In this paper, we use both a $\KL$ penalization w.r.t. the previous policy and a Shannon entropy bonus in a policy iteration context. In~\citet{MDQN}, both types of regularizers were used but in a value iteration context. \citet{Politex, improvedpolitex} are policy iteration methods but only use $\KL$ penalization. %

\textbf{Policy Mirror Descent.}
Prior works on PMD focus mostly on performing a theoretical analysis of convergence speeds or sample complexity for different choices of regularizers~\citep{li2022homotopic,johnson2023optimal,ampo,PMD,Lan,protopapas2024policy}. As PMD provides a general framework for many regularized RL algorithms, PMD theoretical results can be naturally extended to many policy gradient algorithms like natural policy gradient \citep{khodadadian2021linear} and TRPO \citep{TRPO} as shown in \cite{Neu, Geist19}. However, the deep RL algorithms from the PMD family generally %
perform inexact policy updates, adding a source of error from the theoretical perspective. For example, TRPO and the more recent MDPO~\citep{tomar2022mirror} rely on approximate policy updates using policy gradients. %
It was shown in~\cite{PMD}, that an inexact policy update will add an additional error floor independent of the policy evaluation error. By proposing a finite-memory policy update, we provide new convergence results that offer a new deep RL algorithm policy update step that does not introduce any additional policy update error, in contrast to prior works.

\textbf{Ensemble methods and growing neural architectures in RL.} Saving past Q-functions has previously been investigated in the context of policy evaluation. In \citet{tosatto2017boosted}, a first Q-function is learned, then frozen and a new network is added, learning the residual error. \citet{shi2019regularized} uses past Q-functions to apply Anderson acceleration for a value iteration type of algorithm. \citet{adqn} extend DQN by saving the past 10 Q-functions, and using them to compute lower variance target values. Instead of past $Q$-functions, \citet{chen2021randomized,lee2021sunrise,agarwal2020optimistic,lan2020maxmin} use an ensemble of independent $Q$-network functions to stabilize $Q$-function learning in DQN type of algorithms. The aforementioned works are orthogonal to ours, as they are concerned with learning one $Q$, which can all be integrated into StaQ. Conversely, both \citet{PhilippeCascade} and \citet{Micarl} use a special neural architecture called the cascade-correlation network \citep{cascadeNN} to grow neural policies. The former work studies such policies in combination with LSPI~\citep{LSPI}, without entropy regularization. The latter work is closer to ours, using a $\KL$-regularizer but without a deletion mechanism. As such the policy grows indefinitely, limiting the scaling of the method. Finally, \citet{Politex} save the past 10 Q-functions to compute the policy in Eq.~\refp{eq:pol} for the specific case of $\beta = 1$, but do not study the impact of deleting older Q-functions as we do in this paper. Growing neural architectures are common in the neuroevolution community \citep{Neat}, and have been used for sequential decision making tasks, but are beyond the scope of this paper. 

\textbf{Parallels with Continual Learning.} Continual Learning (CL) moves from the usual i.i.d assumption of supervised learning towards a more general assumption that data distributions change through time~\citep{Parisi, CLRobo, DeLangeContSurv, WangContSurv}. This problem is closely related to that of incrementally updating the policy $\pi_k$, due to the changing data distributions that each Q-function is trained on, and our approach of using a growing neural architecture to implement a $\KL$-regularized policy update can be seen as a form of parameter isolation method in the CL literature, which offer some of the best stability-performance trade-offs (see Sec. 6 in~\citet{DeLangeContSurv}). Parameter isolation methods were explored in the context of continual RL~\citep{PNN}, yet remain understudied in a standard \emph{single-task} RL setting. 

\section{Preliminaries}
\label{sec:prelim}
Let a Markov Decision Problem (MDP) be defined by the tuple $(S, A, R, P, \gamma)$, such that $S$ and $A$ are finite state and action spaces, $R$ is a bounded reward function $R:S\times A\mapsto [-\Rx, \Rx]$ for some positive constant $\Rx$, $P$ defines the (Markovian) transition probabilities of the decision process and $\gamma$ is a discount factor. The algorithms presented in this paper can be extended to more general state spaces. However, we limit ourselves to studying finite action spaces, to simplify the sampling from the Boltzmann distribution for the policy (Eq.~\refp{eq:intro:staq}), which would require deeper algorithmic changes for continuous actions spaces, and would remove the benefit of an optimization-free policy update. %

Let $\Delta(A)$ be the space of probability distributions over $A$, and $h$ be the negative entropy given by $h: \Delta( A) \mapsto \mathbb{R}$, $h(p) = p \cdot \log p$, where $\cdot$ is the dot product and the $\log$ is applied element-wise to the vector $p$. Let $\pi :S\mapsto \Delta (A)$ be a stationary stochastic policy  mapping states to distributions over actions. We denote the entropy regularized V-function for policy $\pi$ and regularization weight $\tau > 0$ as $V_{\tau}^{\pi}: S\mapsto \R$, which is defined by 
\begin{align}
    V_{\tau}^{\pi}(s) =\mathbb{E}_{\pi}\left[\sum_{t=0}^{\infty} \gamma ^{t}\{R(s_{t},a_{t}) -\tau h(\pi (s_{t}))\}\middle| s_{0}=s\right].
\end{align}
In turn, the entropy regularized Q-function is given by $Q_{\tau}^{\pi}(s,a) =R(s,a) +\gamma \mathbb{E}_{s'}\left[V_{\tau}^{\pi}(s')\right]$. The V-function {can be written as the expectation of the} Q-function {plus the current state entropy, i.e.} $V_{\tau}^{\pi}(s) =\mathbb{E}_{a}\left[Q_{\tau}^{\pi}(s,a)\right] -\tau h( \pi(s))$ which leads to the Bellman equation 
    $Q_{\tau}^{\pi}(s,a) =R(s,a) +\gamma \mathbb{E}_{s',a'}\left[Q_{\tau}^{\pi }( s',a') - \tau h(\pi (s'))\right]$.
In the following, we will write policies of the form $\pi(s)\propto\exp(Q(s, \cdot))$ for all $s\in S$ more succinctly as $\pi\propto\exp(Q)$. We define optimal V and Q functions where
for all $s\in S, a\in A$, $\Vts(s) := \max_\pi \Vtp(s)$ and  $\Qts(s, a) := \max_\pi \Qtp(s, a)$.
Moreover, the policy $\pi^\star\propto \exp(\Qts/\tau)$ satisfies $Q_{\tau}^{\pi^\star} = \Qts$ and $V_{\tau}^{\pi^\star} = \Vts$ simultaneously for all $s\in S$ \citep{PMD}. In the following, we will overload notations of real functions defined on $S\times A$ and allow them to only take a state input and return a vector in $\R^{|A|}$. For example, $\Qtp(s)$ denotes a vector for which the $i^{\text{th}}$ entry $i\in\{1,\dots,|A|\}$ is equal to $\Qtp(s,i)$. 

For the convergence analysis, we will make use of a matrix representation of the MDP by overloading the notations of $R$ and $P$. Let $I_s:S\mapsto\{1, \dots, |S|\}$ be an arbitrary bijective function that will provide an ordering over the state space, and we let $I_{sa}:S\times A\mapsto \{1, \dots,|S||A|\}$ be an arbitrary bijective function that orders the state-action space. Using these indexing functions, we will overload the notation of the reward function by seeing $R$ as an $|S||A|\times 1$ matrix such that row $I_{sa}(s,a)$ and column $1$ of the matrix $R$ satisfies $R_{(I_{sa}(s,a), 1)} = R(s, a)$. Similarly for the transition function $P$ which we see as an $|S||A|\times |S|$ matrix such that $P_{(I_{sa}(s,a),I_s(s'))}=P(s'|s, a)$. 
A policy $\pi$ will be seen as a $|S|\times|S||A|$ matrix such that $\pi_{(I_s(s'),I_{sa}(s,a))}=\pi(a|s)$ if $s'=s$ and $0$ otherwise. On such matrix representation of the policy we can apply the negative entropy row-wise such that $h(\pi)$ is a $|S|\times 1$ matrix where $h(\pi)_{(I_s(s), 1)} = h(\pi(s))$. Using all of the above notations, we write the Bellman equation associated to policy $\pi$ in matrix notation over a Q-function represented by an $|S||A|\times 1$ matrix that satisfies
$T_\tau^\pi \Qtp = \Qtp = R + \gamma P \left[\pi \Qtp - \tau h(\pi)\right]$.
We also write the Bellman optimality operator on matrices as 
$\Tts f = R + \gamma P \left[\max_{p} p f - \tau h(p)\right]$,
where the maximization $\max_p$ is made row-wise over probability matrices of shape $|S|\times |S||A|$ encoded using the same convention as policy matrices $\pi$ described above.

\section{Policy Mirror Descent and averaging of error}
\label{sec:pmd}
To find $\pis$, we focus on Entropy-regularized Policy Mirror Descent (EPMD) methods \citep{Neu, Politex, improvedpolitex} and notably on those that regularize the policy update with an entropy and $\KL$ term and use an entropy regularized Bellman operator \citep{Lan, PMD}. The EPMD setting discussed here is also similar to the regularized natural policy gradient algorithm on softmax policies of \citet{Cen}. We will put special emphasis in this section on policy evaluation errors and show how convergence of EPMD depends on this error.

\subsection{Entropy regularized value iteration}
\label{sec:epmdvi}
To ease the discussion, let us first consider an approximate value iteration algorithm. Let $\xik: S\times A\mapsto \R$ be the unnormalized log-probability (which we refer to as {logits} for short) of $\pik$, i.e. $\pik \propto \exp(\xik)$.
We define entropy regularized value iteration by the following two steps. 

\textbf{Evaluation step:} let $q_k: S\times A\mapsto \R$ be a sequence of functions such that $q_0 = 0$ and for all $k\geq 0$, $q_{k+1} = \Ttkp q_k + \epsilon_{k+1}$, where $\Ttkp = T_\tau^{\pikp}$ is the Bellman operator associated to policy $\pikp$ and $\epsilon_{k+1}$ represents the evaluation error due to, e.g., knowing only a sample estimate of the Bellman operator or knowing it only on a sub-set of the state-action space. 

\textbf{Policy update step:} letting $\xik=0$, i.e. the first policy is uniform over the action space, for each $q_k$ we update the policy in EPMD by solving the following optimization problem
\begin{align}
    \forall s\in S, \quad \pikp(s) &=\underset{p\in \Delta(A)}{\arg\max}\{q_k(s) \cdot p - \tau h(p) -\eta \KL(p;\pik(s))\} \label{eq:pup} 
\end{align}
where $\KL(p;p')=p\cdot (\log p - \log p')$ and $\eta > 0$ is the $\KL$ regularization weight. 
This update admits the well known closed-form solution given by
\begin{align}
 \xi_{k+1} =\beta \xi_{k} + \alpha q_k,\label{eq:puplog}
\end{align}
where $ \alpha =\frac{1}{\eta + \tau}$ and $\beta = \frac{\eta}{\eta + \tau}$. Let us characterize the convergence of such an algorithm. In the remainder of this paper, we will be interested in bounding the norm $\ninf{\Qts-\tau\xik}$ which we want as small as possible since the logits of the optimal policy are $\frac{\Qts}{\tau}$. Moreover, from a bound over $\ninf{\Qts-\tau\xik}$ we can derive the more common bound over Q-functions since
\begin{lem}
    Let policy $\pik\propto\exp(\xik)$ and $\Qtk$ its Q-function; then $\ninf{\Qts - \Qtk}\leq \frac{2\ninf{\Qts - \tau \xik}}{1-\gamma}$.\label{lem:xitoq}
\end{lem}
Proofs of for all theoretical statements are given in the appendix. The convergence of entropy regularized value iteration is given by the following theorem
\begin{thm}(Convergence of entropy regularized value iteration)
    Letting $E_j = (1-\beta)\sum_{i=1}^{j}\beta^{j-i}\epsilon_{i}$ and $R_m = \Rx + \gamma \tau \log |A|$, we have at iteration $k+1$ that $\ninf{\Qts - \tau \xikp} \leq \gamma^{k+1} \ninf{\Qts} + R_m \sum_{i=0}^{k}\gamma^i\beta^{k-i} + \sum_{i=0}^{k}\gamma^i\ninf{E_{k-i}}$.\label{thm:vi}
\end{thm}
The first term of the upper bound goes to zero as $k\rightarrow\infty$. This term is also found in unregularized value iteration (see, e.g. Theorem 1.12 of \cite{AgarwalRL}) and is due to the contraction property of the Bellman operators. The second term is a constant multiplied by $\sum_{i=0}^{k}\gamma^i\beta^{k-i}$. It can be shown that this sum satisfies
   $\sum_{i=0}^k\gamma^i \beta^{k-i} \leq \max\{\gamma, \beta\}^k (k+1)$,
which goes to zero as $k\rightarrow\infty$. In the limit of $\eta \rightarrow 0$, where we would drop the $\KL$ regularization, $\beta \rightarrow 0$ and $\sum_{i=0}^k\gamma^i \beta^{k-i} \rightarrow \gamma^k$, yielding an error term that goes to zero at the same rate as $\gamma^{k+1} \ninf{\Qts}$. However, as we increase the $\KL$ regularization, $\beta$ approaches 1 and this second term satisfies whenever $\beta > \gamma$,
    $\sum_{i=0}^k\gamma^i \beta^{k-i} = \beta^k \sum_{i=0}^k\left(\frac{\gamma}{\beta}\right)^i \leq \frac{\beta^{k+1}}{\beta-\gamma}$.
While this term still goes to zero as $k\rightarrow\infty$, by increasing the $\KL$ regularization we pay the price of a slower convergence for $\beta > \gamma$.

Finally, the term $\sum_{i=0}^{k}\gamma^i\ninf{E_{k-i}}$ constitutes the error floor of entropy regularized value iteration stemming from the evaluation errors $\epsilon_i$. This error floor might remain above zero even as $k\rightarrow\infty$. In the limit of $\eta\rightarrow 0$, $\ninf{E_j} \rightarrow \ninf{\epsilon_{j}}$ and the error floor will tend to $\sum_{i=0}^{k}\gamma^i\ninf{\epsilon_{j}}$, i.e. a weighted sum of the norms of the {evaluation errors}. However, as we increase the $\KL$ penalization, there is a hope that the evaluation errors will cancel each other in $(1-\beta)\sum_{i=1}^{j}\beta^{j-i}\epsilon_{i}$, leading to a lower value of $\ninf{E_j}$ than if we would only consider the norm of the last error $\ninf{\epsilon_j}$. As such, by increasing $\eta$ we might slow down the convergence rate but potentially lower the error floor $\sum_{i=0}^{k}\gamma^i\ninf{E_{k-i}}$ and return a better final policy. This result is similar to that of \cite{Vieillard20}, except that our algorithm uses a Bellman operator that only applies entropy regularization (as used for example in the learning of Q-functions in SAC \citep{SAC_app}), whereas \cite{Vieillard20} considered the Bellman operator that applies both an entropy and a $\KL$ regularization. Moreover, the latter work was restricted to the analysis of value iteration, but before discussing StaQ we need first to extend the above analysis to policy iteration as StaQ is a policy iteration algorithm.
\subsection{Entropy regularized policy iteration}
\label{sec:epmdpi}
The policy iteration version of EPMD is quite similar to value iteration except that in the evaluation step, $q_k = \Qtk + \epsilon_k$, where $\Qtk = Q_\tau^{\pik}$ is the Q-function associated to $\pik$. The approximation $q_k$ of $\Qtk$ can be obtained for instance by applying the Bellman operator $\Ttk$ (or a noisy version thereof) several times on $q_{k-1}$---instead of a single time in value iteration. As value and policy iteration algorithms are quite similar, the analysis of the latter follows the same general template, except that the error propagates in slightly more complex ways. In the case of policy iteration, we need a way to relate $q_k$ and $q_{k+1}$ as this relation is not as direct as in value iteration. To do so, we will make use of the policy improvement lemma (Section 4.2 of \cite{SuttonRL}). In the unregularized RL case, the lemma states that a policy greedy w.r.t. a Q-function will have a greater or equal Q-function at every state-action pair. A similar property holds in the entropy regularized case. However, in the presence of policy evaluation errors, the new policy $\pikp$---obtained by maximizing Eq.~\ref{eq:pup}---might increase the probability of sub-par actions and improvement is only guaranteed up to some policy improvement error characterized by the following lemma
\begin{lem}[Approximate policy improvement]
    Let $\mu_{k} = (I-\gamma P\pik)^{-1}$ be the (unnormalized) state distribution associated to policy $\pik$. At any iteration $ k\geq 0$ of entropy regularized policy iteration, we have that $Q_{\tau}^{k+1}\geq Q_{\tau}^{k} - \epik$, with $\epik:=\gamma \mu_{k+1}P(\pik-\pikp)\epsilon_k$. \label{lem:pi}
\end{lem}
The policy improvement error $\epik$ is invariant to a constant shift in the evaluation error $\epsilon_k$. Indeed, we have that for any real value $c$, $(\pik-\pikp)(\epsilon_k+c\mathbf{1}) = (\pik-\pikp)\epsilon_k$, where $\mathbf{1}$ is a vector of ones. Additionally, if $\epsilon_k = 0$ or any other constant vector, then policy improvement is guaranteed. However, we might not improve over the previous Q-function if we overestimate a bad action or underestimate a good one. The overall convergence of entropy regularized policy iteration is characterized by the following theorem
\begin{thm}(Convergence of entropy regularized policy iteration)
    Letting $\epsilon_{-1} = 0$ and $E_j := (1-\beta)\sum_{i=0}^{j}\beta^{j-i}(\epsilon_{i}-\gamma P\pi_i(\epsilon_{i-1} + \epsilon_{\Delta_i}))$, we have at iteration $k+1$ that $\ninf{\Qts - \tau \xikp} \leq \gamma^{k+1} \ninf{\Qts} + \frac{2-\gamma -\beta}{1-\gamma}R_m\sum_{i=0}^{k}\gamma^i\beta^{k-i} + \sum_{i=0}^{k}\gamma^i\ninf{E_{k-i}}$.\label{thm:pi}
\end{thm}
As can be seen, the upper bound of Theorem~\ref{thm:pi} follows a very similar structure to that of value iteration, with the main difference being in the error floor $\sum_{i=0}^{k}\gamma^i\ninf{E_{k-i}}$ that now notably depends on the policy improvement error discussed above. While this error floor involves more quantities, the general scheme remains the same and one hopes that there are values of $\eta$ such that a cancellation of terms leads to a lower error floor compared to the unregularized case while not slowing policy iteration too much. This analysis improves over that of \cite{PMD}, that only considered a uniform worst case error, leading to a less interesting upper bound where the smallest error floor---and the fastest convergence rate---is always obtained by choosing $\eta=0$, using no $\KL$ regularization.

\subsection{Approximate policy update}
\label{sec:prelim:approxpolup}
We have analyzed so far the convergence of EPMD algorithms, considering only evaluation errors. In practice, the policy update that consists in solving Eq.~\ref{eq:pup} might prove challenging without approximations. Indeed, while this policy update leads to a closed form solution in the space of policy logits (Eq.~\ref{eq:puplog}), it might not be possible to implement exactly if the state-action space is too large. In this case, one would typically use a second function approximator for the policy. Such actor-critic approaches to PMD include the likes of TRPO \citep{TRPO}, ECPO \cite{ECPO} and MDPO \citep{tomar2022mirror} that are discussed in more details in App.~\ref{sec:app:acpmd}.

Irrespective of the AC method used, representing the policy logits with a function approximator is likely to introduce a new type of error. In \cite{PMD}, the authors analyzed an approximate EPMD scheme such that the policy update objective in Eq.~\ref{eq:pup} is optimized up to some error $\epsilon_{\text{opt}}$ for all states and all iterations. They showed that the resulting algorithm would converge at the same rate as its exact counterpart but would reach an error floor that depends on $\epsilon_{\text{opt}}$, independently of the existence of policy evaluation errors. In this paper, we investigate an alternative policy update, that truncates the infinite sum of EPMD to result in a practical algorithm that does not introduce errors from the policy update yet keeps the appealing property of averaging policy evaluation errors.

\section{Finite-memory policy mirror descent}
\label{sec:fmdpg}
Let us now consider a PMD-like algorithm that keeps in memory at most $M$---with $M$ being a finite and strictly positive integer---Q-function estimates. The policy at iteration $k$ is now given by Eq.~\ref{eq:intro:staq}, which can be written as a recursive update in the logits space in the following way
\begin{align}
    \xi_{k+1} =\beta \xi_{k} + \alpha q_k +  \frac{\alpha \beta^M}{1-\beta^M} (q_k - q_{k-M}),\label{eq:ximem2}
\end{align}
where $q_{k-M} := 0$ whenever $k-M<0$, and $q_k = \Qtk + \epsilon_k$ otherwise. In contrast to vanilla EPMD in Eq.~\ref{eq:puplog}, we now `delete' at each update the oldest Q-function estimate $q_{k-M}$ and also slightly overweight the most recent Q-function estimate to ensure that the Q-function weights sum to 1. This weight correction in Eq.~\ref{eq:intro:staq}---the extra multiplication by $\frac{1}{1-\beta^M}$ compared to the vanilla sum in Eq.~\ref{eq:pol}---is important as otherwise, the logits might never converge to $\frac{\Qts}{\tau}$ even when the last $M$ Q-function estimates are all equal to $\Qts$. Indeed, without the weight correction and since $\tau\alpha = 1-\beta$, we would have $\tau\xik =  (1-\beta)\sum_{i=0}^{M-1}\beta^i\Qts = ({1-\beta^M}) \Qts$.

The logits update in Eq.~\ref{eq:ximem2} can be interpreted as the result of the following optimization problem
\begin{equation}
    \forall s\in S, \pikp(s) =\underset{p\in \Delta(A)}{\arg\max}\ p\cdot[q_k+\frac{\beta^M}{1-\beta^M} (q_k - q_{k-M})](s) - \tau h(p) -\eta \KL(p;\pik(s)) \label{eq:fmpup} 
\end{equation}
Now instead of maximizing the latest Q-function estimate, the policy also maximizes the difference between the latest and oldest estimate out of the last $M$ Q-functions. This introduces an additional source of policy improvement error, but in our theoretical analysis, we show that for a finite but large enough $M$, this error will vanish as $k\rightarrow\infty$, leaving us with an error floor that only depends on the evaluation errors. Specifically, the algorithm given by Eq.~\ref{eq:ximem2} has the following convergence properties
\begin{thm}[Convergence of finite memory EPMD]
    Let $M$ such that $M > \log \frac{(1-\gamma)^3}{(1+\gamma)(1-\gamma)^2 + 4(\gamma+\gamma^2)}/\log{\beta}$, $\gamma_M = \frac{\gamma}{1-\beta^M}$, $c =  \frac{\beta^M}{1-\beta^M} \left(\frac{4(\gamma+\gamma^2)}{(1-\gamma)^2} + \gamma\right)$, let $d_0$ be the unique root of $d^{2M+1} - \gamma_M d^{2M} - c$ in the interval $(\gamma_M, 1)$, define the matrix $A_{k+1} = \gamma P \pikp (I+ \gamma \mu_{k+1}P(\pikp-\pik))$, and error terms $E_j =  \frac{1-\beta}{1-\beta^M}\sum_{i=0}^{M-1}\beta^{i}(\epsilon_{k-i}- A_{k-i} \epsilon_{k-i-1})$, and $T_k = \ninf{E_k} + \frac{\beta^M(1-\beta)}{(1-\beta^M)^2}\ninf{\sum_{i=0}^{M-1}\beta^{i}A_{k-i}(\epsilon_{k-i-1} - \epsilon_{k-i-1-M})}$ and worst error term $\bar{T} = \underset{0\leq i\leq k}{\max} T_i$, then $\ninf{\Qts-\xikp}\leq d_0^{k+1}\ninf{\Qts} + \sum_{i=0}^{k} \gamma_M^{i} \left(T_{k-i} + c \frac{\bar{T}}{1-\gamma_M-c}\right)$.\label{thm:fmpi}
\end{thm}
In finite memory EPMD we no longer have an explicit expression for the convergence rate, but provided $M$ is large enough, we know that it exists as the unique root of a function in the range $(\gamma_M, 1)$. Moreover, as $M\rightarrow\infty$, the convergence rate goes to $\gamma$ since in this case: $\gamma_M\rightarrow \gamma$, $c\rightarrow 0$ and $\gamma$ becomes quite clearly the unique non-zero solution to $d^{2M+1} - \gamma d^{2M} = 0$ for unknown $d$. As in exact EPMD, we note that in the absence of evaluation errors, the algorithm converges to the optimal policy. In terms of error floor, we find a similar expression of $E_j$ as in Sec.~\ref{sec:epmdpi} up to the truncation to the latest $M$ errors. However, the error floor does not depend only on $E_j$ but on an additional term in $T_j$ that depends on older evaluation error terms. This additional term is weighted by $\frac{\beta^M(1-\beta)}{(1-\beta^M)^2}$ which decreases exponentially fast towards zero as we increase $M$. Finally, the error floor is not just a weighted sum of average evaluation errors but also depends on the worst average evaluation error $\bar{T}$. Fortunately, this term is weighted again by a constant that decreases exponentially fast towards zero as $M$ increases. Thus the theoretical analysis indicates in terms of averaging of policy evaluation errors, that we approach the desirable behavior of an exact EPMD policy update exponentially fast by increasing $M$.

The rest of the paper explores the practical implications of this result: what are the tangible benefits of EPMD, and how fast (in terms of memory size $M$) can a finite memory variant approach the behavior of EPMD with an exact policy update? We first need an efficient implementation of EPMD.

\subsection{Practical implementation}
\label{sec:stack}
\begin{figure}
\begin{minipage}{0.56\textwidth}
    \centering
    \includegraphics[width=\linewidth]{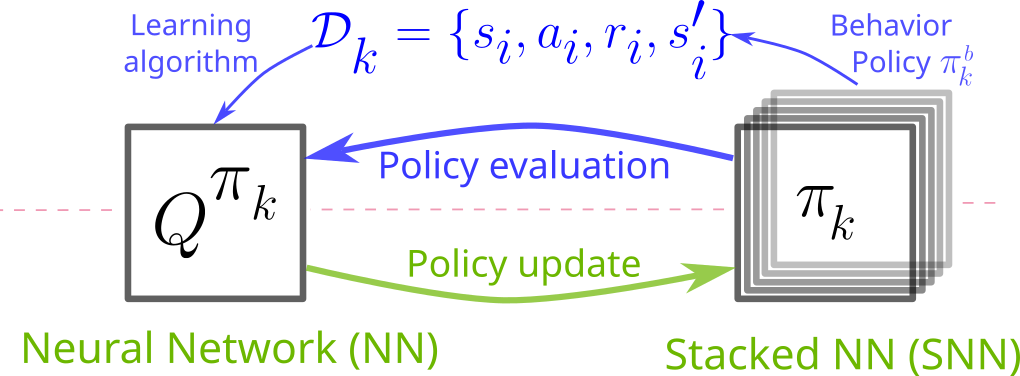}
\end{minipage}\hfill
\begin{minipage}{0.4\textwidth}
    \small
    \centering
    \begin{tabular}{lr}
    \toprule
    Algorithm & Runtime (hr) \\
    \midrule
    StaQ M=1                  & 9.6 \\
    StaQ M=100                & 10.4 \\
    StaQ M=300                & 10.5 \\
    \hline
    AC-NoKL (SAC)        & 13.9 \\
    AC-MProj (ECPO)      & 13.1 \\
    AC-DirectOpt (MDPO)  & 13.2 \\
    \hline
    DQN & 5.8 \\
    \hline
    \end{tabular}
\end{minipage}
\caption{\textbf{Left:} Overview of StaQ, showing the continual training of a Q-function (left), from which we periodically ``stack'' frozen weight snapshots to form the policy (right). At each iteration $k$, we perform: i) Policy evaluation, generating a dataset ${\cal D}_k$ of transitions that are gathered by a behavior policy $\pi_k^b$, typically derived from $\pi_k$, and then learn $Q^{\pi_k}$ from ${\cal D}_k$; ii) Policy update, performed by ``stacking'' the NN of $Q^{\pi_k}$ into the current policy. The policy update is optimization-free and theoretically grounded (Sec.~\ref{sec:fmdpg}), thus only the choice of $\pi_k^b$ and the policy evaluation algorithm can remain sources of instabilities in this deep RL setting. \textbf{Right:} Runtime of StaQ and Actor-Critic baselines on MinAtar/SpaceInvaders-v1 (NVIDIA Tesla V100 GPU). Increasing $M$ has a marginal effect on run time, and StaQ is faster than AC methods due to the optimization-free policy update. DQN is $\sim$65\% faster than $M=1$ StaQ, as it only uses one target Q-function, rather than two.}\label{fig:overview}
\end{figure}
We implement an efficient version of the policy in Eq.~\ref{eq:intro:staq} using stacked neural networks (illustrated in Fig.~\ref{fig:overview}). By using batched operations we make efficient use of GPUs and compute multiple Q-values in parallel. We call the resulting algorithm StaQ\footnote{The code is available at~\href{https://github.com/alexdavey/StaQ}{github.com/alexdavey/StaQ}.}. After each policy evaluation, we push the weights corresponding to this new Q-function onto the stack. If the stacked NN contains more than $M$ NNs, the oldest NN is deleted in a ``first in first out'' fashion. If implementing exact EPMD, then we never delete older Q-functions.

To further reduce the impact of a large $M$, we pre-compute $\xik$ for all entries in the replay buffer\footnote{Since we use small replay buffer sizes of 50K transitions, we are likely to process each transition multiple times (25.6 times in expectation in our experiments) making this optimization worthwhile.} at the start of policy evaluation. The logits $\xik$ are used to sample on-policy actions when computing the targets for $\Qtk$. As a result of the pre-computation, during policy evaluation, forward and backward passes only operate on the current Q-function and hence the impact of large $M$ is minimized. However rolling out the current behavioural policy $\pi_k^b$ still requires a full forward pass. Conversely, the policy update consists only of adding the new weights to the stack, and thus, is optimization free and (almost) instantaneous. The table in Fig~\ref{fig:overview} shows that varying $M$ has little impact on the runtime of StaQ on GPU (requiring at most 3GB of GPU memory), at least for medium-sized environments such as MinAtar \citep{Minatar}. StaQ is faster than the Actor-Critic baselines (described in App.~\ref{sec:app:acpmd}) as no gradient steps are required for the policy update, and within the same order of magnitude as fast Deep RL baselines such as DQN -- the speed difference being due to the use of two target Q-functions. Further implementation details are provided in App.~\ref{app:hyperparams}.

\section{Experiments}
\label{sec:xp}
\textbf{Environments.} For the main experiments we consider the 5 MinAtar environments~\citep{Minatar}. We supplement this, in the Appendix, with results for 4 Classic Control tasks from Gymnasium~\citep{gym}, covering the full suite of environments suggested by~\citet{RevRainbow} for comparing deep RL algorithms with finite action spaces.

\textbf{Algorithms.}
In our main set of experiments, we compare finite memory EPMD for several values of $M$ on up to 5 million timesteps. We notably consider the two extremes of $M=1$, using no $\KL$ regularization and $M=1000$ that never deletes a Q-function within the 5 million timesteps window~(labeled \texttt{Exact~PMD} in figures). We further add several baselines below, based on policy updates common in the literature, see App.~\ref{app:connection_baselines} for more details. The policy evaluation strategy, and all other hyperparameters are kept the same as StaQ, such that they only differ by their policy update.

\texttt{NatGrad~+~LS~(TRPO)} uses the conjugate gradient approach of TRPO~\citep{TRPO}, with line search. We also add several actor-critic methods, including \texttt{AC-DirectOpt (SAC)}, a discrete-actions version of SAC~\citep{SAC_app}. Note that in the discrete-actions setting we can sample directly from the policy and so are not required to introduce an actor --- the $M=1$ limit of StaQ can be seen as a discrete-actions version of SAC (see App.~\ref{app:notsac}).

Two further actor-critic algorithms which optimize the PMD objective are: i) ECPO~(\texttt{AC-MProj}, \cite{ECPO}) which minimizes the M-Projection between the actor and the recursive PMD solution $\pi_{k+1}\propto\pi_k^\beta \exp\left({\alpha\Qtk}\right)$, argued to better preserve the support of the distribution and prevent the premature elimination of actions, and ii) MDPO (\texttt{AC-DirectOpt}, ~\cite{tomar2022mirror}, which skips the closed form solution, and directly optimizes a state averaged version of the entropy regularized PMD update by gradient ascent.

\textbf{Deep RL baselines.} In the appendix, we complement our main experiment with comparisons to deep RL baselines such as the value iteration algorithm DQN~\citep{dqn} and its entropy-regularized variant M-DQN~\citep{MDQN}, the policy gradient algorithm TRPO~\citep{TRPO} and PPO~\citep{PPO}. %
These baselines differ more widely and in ways orthogonal (policy evaluation, exploration, replay buffer management) to the main focus of this paper which is the policy update of EPMD, and thus are only provided for reference.

\begin{wrapfigure}{r}{0.42\textwidth}
    \centering
    \vspace{-1\baselineskip} 
    \includegraphics[width=\linewidth]{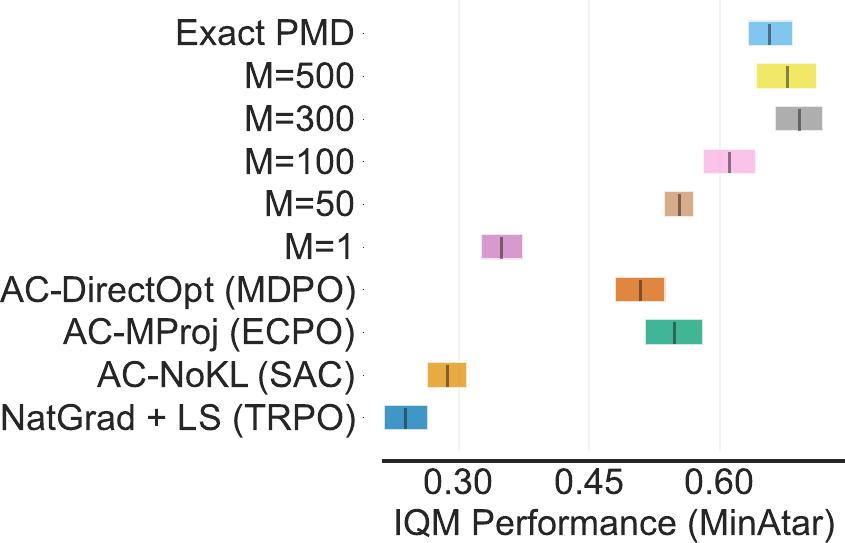}  
    \caption{Normalized policy performance aggregated across MinAtar tasks, showing the interquantile mean and 95\% confidence intervals. See App.~\ref{app:experimental_results} for further results and Deep RL baselines.}
    \label{fig:mainxp_m}
    \vspace{-1\baselineskip} 
\end{wrapfigure}
\textbf{Results.} Across 20 seeds per algorithm and task, aggregate performance improves almost monotonically with $M$ (Fig.~\ref{fig:mainxp_m}). On nearly every task, sufficiently large $M$ matches exact EPMD, with $M\ge300$ virtually indistinguishable on all tasks (App.~\ref{app:Mablation}). This reinforces the theoretical insights that StaQ can match the behavior of exact EPMD.

Compared to the natural policy gradient and other actor-critic PMD baselines, performance is generally improved across all tasks, indicating that StaQ is potentially a better alternative to other actor-critic based PMD schemes, at least on environments where Q-function computations can be efficiently batched, because of the improved performance and runtime. In the appendix, we provide additional experimental results, including comparisons to deep RL baselines in Fig.~\ref{fig:xp_returns}. In App.~\ref{sec:app:shinrl}, we use the ShinRL library~\citep{shinrl} to investigate the cancellation of evaluation errors on environments where Q-functions can be computed exactly, and show that it can occur even on simple environments.

\section{Conclusion}
\label{sec:future}
In this paper, we proposed a policy update rule based on policy mirror descent, that keeps in memory at most $M$ Q-functions. By increasing $M$ we can quickly mimic exact EPMD both from a theoretical and empirical perspective. The resulting policy update has a solid theoretical foundation and clear empirical benefits, making it a valid alternative to existing EPMD schemes, at least for medium-sized tasks. Due to its exact policy update, and the absence of a gap between the theoretical algorithm and the practical implementation, StaQ provides a promising setting for testing other components of RL such as policy evaluation, for instance, the recent methods that use normalization techniques to reduce policy evaluation error~\citep{gallici2025, bhatt2024}. 

Surprisingly, even when $M$ is large, the final computational burden on many common RL tasks is small, due to the stacking of the Q-functions which parallelize well on modern hardware. For large-scale tasks such as Atari~\citep{Atari}, the memory usage and inference time may become significant, however future algorithmic improvements such as compressing old Q-functions could yield a computationally efficient deep RL algorithm, at the cost of an inexact policy update.

\section*{Acknowledgments}
A. Davey and B. Driss were funded by the project ANR-23-CE23-0006. This work was granted access to the HPC resources of IDRIS under the allocation 2024-AD011015599 made by GENCI.

\clearpage



\bibliography{bibliography}
\bibliographystyle{rlj}

\beginSupplementaryMaterials


\section{Proofs}
\label{sec:proofs}
This section includes proofs of the lemmas and theorems of the main paper. 
\subsection{Properties of entropy regularized Bellman operators}
We first start with a reminder of some basic properties of the (entropy regularized) Bellman operators, as presented in \citep{Geist19,PMD}. Within the MDP setting defined in Sec.~\ref{sec:prelim}, let $ T_{\tau}^{\pi}$ be the operator defined for any map $f:S\times A\mapsto \mathbb{R}$ by 
\begin{align}
\left( T_{\tau}^{\pi} f\right)(s,a) =R(s,a) +\gamma \mathbb{E}_{s',a'}[f(s',a') -\tau h(\pi(s'))] ,\label{eq:Tpidef}
\end{align}
This operator has the following three properties.
\begin{prop}[Contraction]
    $T_{\tau}^{\pi}$ is a $\gamma$-contraction w.r.t. the $\ninf{.}$ norm, i.e. $\ninf{T_{\tau}^{\pi} f - T_{\tau}^{\pi} g} \leq \gamma \ninf{f - g}$ for any real functions $f$ and $g$ of $S\times A$.
\end{prop}
\begin{prop}[Fixed point]
\label{prop:softfp}
    $Q_{\tau}^\pi$ is the unique fixed point of the operator $T_{\tau}^{\pi}$, i.e. $T_{\tau}^{\pi} Q_{\tau}^\pi = Q_{\tau}^\pi$.
\end{prop}
Let $f$, $g$ be two real functions of $S\times A$. We say that $f\geq g$ iff $f(s,a)\geq g(s,a)$ for all $(s,a)\in S\times A$.
\begin{prop}[Monotonicity]
\label{prop:softmon}
$T_{\tau}^{\pi}$ is monotonous, i.e. if $ f\geq g$ then $T_{\tau }^{\pi } f\geq T_{\tau}^{\pi} g$.    
\end{prop}
Let the Bellman optimality $\Tts$ operator be defined by
\begin{align}
\left( \Tts f\right)(s,a) = R(s,a) + \gamma \mathbb{E}_{s'}\left[\max_{p\in\Delta(A)}f(s')\cdot p -\tau h(p)\right].
\end{align}
For the Bellman optimality operator we need the following two properties. 
\begin{prop}[Contraction]
    $\Tts$ is a $\gamma$-contraction w.r.t. the $\ninf{.}$ norm, i.e. $\ninf{\Tts f - \Tts g} \leq \gamma \ninf{f - g}$ for any real functions $f$ and $g$ of $S\times A$.
\end{prop}
\begin{prop}[Optimal fixed point]
    $\Tts$ admits $\Qts$ as a unique fixed point, satisfying $\Tts \Qts = \Qts$.
\end{prop}
Finally, we will make use of the well known property that the softmax distribution is entropy maximizing \citep{Geist19}. Specifically, we know that the policy $\pik\propto\exp(\xik)$  satisfies the following property
\begin{align}
    \text{for all } s\in S, \quad \pik(s) &= \underset{p\in\Delta(A)}{\arg\max}\ \xik(s) \cdot p - h(p).   \label{eq:piboltzmax}
\end{align}
\subsection{Proof of Lemma \ref{lem:xitoq}}
\begin{proof}
We first observe from the definition of $\pik$ that
    \begin{align}
        \Ttk \tau \xik &= R + \gamma P (\pik \tau \xik - \tau h(\pik)),\\
        &\overset{(i)}{=} R + \gamma P (\max_p p \tau \xik - \tau h(p)),\\
        &= \Tts \tau\xik,
    \end{align}
with $(i)$ due to Eq.~\ref{eq:piboltzmax}. Then
    \begin{align}
        \Qtk - \tau \xik &= (\Ttk \Qtk - \Ttk \tau \xik) + (\Ttk \tau \xik - \tau \xik)\\
        &= \gamma P_k (\Qtk - \tau \xik) + (\Ttk \tau \xik - \tau \xik),\\
        &= (I-\gamma P_k)^{-1} (\Ttk \tau \xik - \tau \xik),\\
        &= (I-\gamma P_k)^{-1} (\Tts \tau \xik - \tau \xik),\\
        \Rightarrow \ninf{\Qtk - \tau \xik} &\leq \frac{1}{1-\gamma} \ninf{\Tts \tau \xik - \tau \xik},\\
        &= \frac{1}{1-\gamma} \ninf{\Tts \tau \xik - \Qts + \Qts - \tau \xik},\\
        &\leq \frac{1}{1-\gamma} \left(\ninf{\Tts\Qts - \Tts\tau \xik} + \ninf{\Qts - \tau \xik}\right),\\
        &\leq \frac{1+\gamma}{1-\gamma}  \ninf{\Qts - \tau \xik}.
    \end{align}
Finally, 
\begin{align}
    \ninf{\Qts - \Qtk} &\leq \ninf{\Qts - \tau \xik} + \ninf{\Qtk - \tau \xik},\\
    &\leq \ninf{\Qts - \tau \xik} + \frac{1+\gamma}{1-\gamma}  \ninf{\Qts - \tau \xik} ,\\
    &= \frac{2 \ninf{\Qts - \tau \xik}}{1-\gamma}.
\end{align}

\end{proof}
\subsection{Proof of Theorem \ref{thm:vi}}
\dg{We prove here a more general version of Theorem~\ref{thm:vi}, by taking into account a policy update error at each iteration $k$ given by the function $\epsilon'_k:S\times A\mapsto \R$, defined by $\xikp = \beta \xik + \alpha q_k + \epsilon'_k$, and using the same matrix notation as those of real function of the state-action space introduced in Sec.~\ref{sec:prelim}. To retrieve vanilla Theorem~\ref{thm:vi} as discussed in the main text, one can simply ignore the extra colored terms.}
\begin{proof}
Extracting the value function from the policy update formula, and using the fact that $\tau \alpha = 1 - \beta$, we have
    \begin{align}
        \pikp q_k - \tau h(\pikp) &=  \pikp \left(\frac{1}{1-\beta}(\tau\xikp - \beta \tau \xik)\dg{-\frac{\epsilon'_k}{\alpha}}\right) - \tau h(\pikp),\\
        &= \pikp \left(\frac{1}{1-\beta}(\tau\xikp - \beta \tau \xik)\dg{-\frac{\epsilon'_k}{\alpha}}\right) - \frac{1}{1-\beta}(\tau h(\pikp) - \beta \tau h(\pikp)),\\
        &= \pikp \left(\frac{1}{1-\beta}[\tau\xikp - \tau h(\pikp) - \beta (\tau \xik - \tau h(\pikp))]\right) \dg{- \pikp \frac{\epsilon'_k}{\alpha}},\\ 
        &\overset{(i)}{\geq} \left(\frac{1}{1-\beta}[\pikp\tau\xikp - \tau h(\pikp) - \beta (\pik\tau \xik - \tau h(\pik))]\right) \dg{-\pikp \frac{\epsilon'_k}{\alpha}}, \label{eq:qkdk}
    \end{align}
with $(i)$ due to $\pik\tau \xik - \tau h(\pik) = \max_p p\tau \xik - \tau h(p) \geq \pikp\tau \xikp - \tau h(\pikp)$. Using this inequality in $\qkp$ yields

    \begin{align}
        \qkp &= \Ttkp q_k + \epsilon_{k+1},\\
        &= R + \gamma P(\pikp q_k - \tau h(\pikp)) + \epsilon_{k+1},\\
        \begin{split}            
        &\geq R + \gamma P\Big(\frac{1}{1-\beta}[\pikp\tau\xikp  - \tau h(\pikp) - \\
        &\quad\quad\beta (\pik \tau\xik - \tau h(\pik))]\dg{- \pikp \frac{\epsilon'_k}{\alpha}}\Big) + \epsilon_{k+1},
        \end{split}\\
        \begin{split}
            &= \frac{1}{1-\beta}(R - \beta R) + \gamma P\Big(\frac{1}{1-\beta}[\pikp\tau\xikp  - \tau h(\pikp)- \\
            &\quad\quad\beta (\pik \tau\xik - \tau h(\pik))]\dg{- \pikp \frac{\epsilon'_k}{\alpha}}\Big) + \epsilon_{k+1},
        \end{split}\\
        &= \frac{1}{1-\beta} (\Ttkp \tau \xikp - \beta \Ttk \tau \xik) + \epsilon_{k+1}\dg{- \gamma P_{k+1} \frac{\epsilon'_k}{\alpha}},\\
        &= \frac{1}{1-\beta} (\Tts \tau \xikp - \beta \Tts \tau \xik) + \epsilon_{k+1}\dg{- \gamma P_{k+1} \frac{\epsilon'_k}{\alpha}},
    \end{align}
with \dg{$P_k = P\pikp$}, and where the last step is again due to Eq.~\ref{eq:piboltzmax}. Now using this inequality in the definition of $\xikp$ gives
\begin{align}
    \tau \xikp &= (1-\beta) \sum_{i=0}^{k} \beta^{k-i} q_{i} \dg{ + (1-\beta)\sum_{i=0}^{k} \beta^{k-i} \epsilon'_{i}},\\
    &\overset{(i)}{=} (1-\beta) \sum_{i=1}^{k} \beta^{k-i} q_{i}\dg{ + (1-\beta)\sum_{i=1}^{k} \beta^{k-i} \epsilon'_{i}},\\
        &\geq (1-\beta) \sum_{i=1}^{k} \beta^{k-i} \left( \frac{1}{1-\beta}(\Tts \tau \xi_i - \beta \Tts \tau \xi_{i-1}) + \epsilon_i\right) \dg{\quad \quad  + (1-\beta)\sum_{i=1}^{k} \beta^{k-i} \Big(\epsilon'_{i}-\gamma P_i \frac{\epsilon'_{i-1}}{\alpha}\Big)},\\
    &= \Tts \tau \xik - \beta^{k} \Tts \tau \xi_0 + (1-\beta)\sum_{i=1}^{k}\beta^{k-i}\epsilon_{i} \dg{ + (1-\beta)\sum_{i=1}^{k} \beta^{k-i} \Big(\epsilon'_{i}-\gamma P_i \frac{\epsilon'_{i-1}}{\alpha}\Big)},
\end{align}
with (i) due to $q_0=0$\dg{ (and $\epsilon'_0 = 0$)}. Letting $E_j = (1-\beta)\sum_{i=1}^{j}\beta^{j-i}\epsilon_{i}$, \dg{$E'_j= (1-\beta)\sum_{i=1}^{j} \beta^{j-i} \Big(\epsilon'_{i}-\gamma P_i \frac{\epsilon'_{i-1}}{\alpha}\Big)$}, and $R_m = R_x + \gamma \tau \log |A|$ be an upper bound to $\ninf{\Tts \tau\xi_0}$, we finally obtain
\begin{align}
    \Qts - \tau \xikp &\leq \Qts - \Tts \tau  \xik + \beta^{k} \Tts  \tau\xi_0 - E_k \dg{-E'_k},\\
    \Rightarrow \ninf{\Qts - \tau \xikp} &\leq \ninf{\Qts - \Tts \tau  \xik} + \beta^{k}\ninf{\Tts  \tau \xi_0} + \ninf{E_k},\\
    &\leq \gamma \ninf{\Qts - \tau \xik} + \beta^{k}R_m + \ninf{E_k}\dg{+\ninf{E'_k}},\\
    &\leq \gamma^{k+1} \ninf{\Qts} + R_m\sum_{i=0}^{k}\gamma^i\beta^{k-i} + \sum_{i=0}^{k}\gamma^i\ninf{E_{k-i}}\dg{+\sum_{i=0}^{k}\gamma^i\ninf{E'_{k-i}}}.
\end{align}
if $\beta > \gamma$
\begin{align}
    \sum_{i=0}^k\gamma^i \beta^{k-i} &= \beta^k \sum_{i=0}^k\left(\frac{\gamma}{\beta}\right)^i,\\
    &= \frac{\beta^{k+1} - \gamma^{k+1}}{\beta-\gamma}.
\end{align}
if $\beta < \gamma$
\begin{align}
    \sum_{i=0}^k\gamma^i \beta^{k-i} &= \sum_{i=0}^k\gamma^{k-i} \beta^{i},\\
    &=\frac{\gamma^{k+1} - \beta^{k+1}}{\gamma-\beta}.
\end{align}
if $\beta = \gamma$
\begin{align}
    \sum_{i=0}^k\gamma^i \beta^{k-i} &= \gamma^k (k+1).
\end{align}
In all cases
\begin{align}
    \sum_{i=0}^k\gamma^i \beta^{k-i} &\leq \max\{\gamma, \beta\}^k (k+1).
\end{align}
\end{proof}

\subsection{Proof of Lemma \ref{lem:pi}}
\begin{proof}
As $\pikp$ maximizes the policy update Eq.~\ref{eq:pup}, and from the non-negativity of the $\KL$ and the fact that $\KL(\pik;\pik)=0$ we have
    \begin{align}
        \pik q_k - \tau h(\pik) &\leq  \pikp q_k - \tau h(\pikp) - \eta \KL(\pikp; \pik),\\
        &\leq  \pikp q_k - \tau h(\pikp),\\
        \Leftrightarrow \pik \Qtk - \tau h(\pik) &\leq  \pikp \Qtk - \tau h(\pikp) + (\pikp - \pik) \epsilon_k.
    \end{align}
    \begin{align}
        \Qtkp - \Qtk &= \gamma P  (\pikp\Qtkp - \tau h(\pikp)) - \gamma P (\pik\Qtk - \tau h(\pik)),\\
        \begin{split}
        &\overset{(i)}{\geq} \gamma P  (\pikp\Qtkp - \tau h(\pikp)) -\\
        &\quad\quad \gamma P (\pikp \Qtk - \tau h(\pikp) + (\pikp - \pik) \epsilon_k),
        \end{split}\\
        &= \gamma P \pikp (\Qtkp - \Qtk) + \gamma P (\pik - \pikp) \epsilon_k,\\
        \Leftrightarrow (I-\gamma P \pikp)(\Qtkp - \Qtk) &\geq \gamma P (\pik - \pikp) \epsilon_k,\\
        \Leftrightarrow \Qtkp - \Qtk &\overset{(ii)}{\geq} \gamma (I-\gamma P\pikp)^{-1}P (\pik - \pikp) \epsilon_k.
    \end{align}
    where in $(i)$ we have used the fact that $P$ is a probability matrix with only positive entries, and similarly in $(ii)$ for the matrix $(I-\gamma P \pikp)^{-1}=\sum_{i=0}^\infty (\gamma P\pikp)^i$.
\end{proof}
\subsection{Proof of Theorem \ref{thm:pi}}
\begin{proof}
The beginning of the proof is the same as in value iteration and we can show using the same arguments that    
    \begin{align}
        \pikp q_k - \tau h(\pikp) &\geq \left(\frac{1}{1-\beta}[\pikp\tau\xikp - \tau h(\pikp) - \beta (\pik\tau \xik - \tau h(\pik))]\right).
    \end{align}
    Using this inequality in $\qkp$ yields
    \begin{align}
        \qkp &= \Qtkp + \epsilon_{k+1},\\
        &= R + \gamma P(\pikp \Qtkp - \tau h(\pikp)) + \epsilon_{k+1},\\
        &\overset{(i)}{\geq} R + \gamma P(\pikp (q_k - \epsilon_k - \epik) - \tau h(\pikp)) + \epsilon_{k+1},\\
        &\geq R + \gamma P\left(\frac{1}{1-\beta}[\pikp\tau\xikp  - \tau h(\pikp) - \beta (\pik \tau\xik - \tau h(\pik))]\right) + \epsilon_{k+1} - \gamma P_{k+1}(\epsilon_k + \epik),\\
        &= \frac{1}{1-\beta} (\Ttkp \tau \xikp - \beta \Ttk \tau \xik) + \epsilon_{k+1} - \gamma P_{k+1}(\epsilon_k + \epik),\\
        &= \frac{1}{1-\beta} (\Tts \tau \xikp - \beta \Tts \tau \xik) + \epsilon_{k+1} - \gamma P_{k+1}(\epsilon_k + \epik),
    \end{align}
where for $(i)$ we used Lemma~\ref{lem:pi} and the definition of $q_k$. Now using this inequality in the definition of $\xikp$ gives
\begin{align}
    \tau \xikp &= (1-\beta) \sum_{i=0}^{k} \beta^{k-i} q_{i},\\
    &\geq (1-\beta) \beta^k q_0 + (1-\beta) \sum_{i=1}^{k} \beta^{k-i} \left( \frac{1}{1-\beta}(\Tts \tau \xi_i - \beta \Tts \tau \xi_{i-1}) + \epsilon_i - \gamma P_{i}(\epsilon_{i-1} + \epiim)\right),\\
    &= \Tts \tau \xik - \beta^{k} \Tts \xi_0 + (1-\beta)\sum_{i=0}^{k}\beta^{k-i}(\epsilon_{i} - \epsilon'_{i}) + (1-\beta) \beta^k Q_\tau^0,
\end{align}
with $\epsilon'_0 = 0$ and $\forall i > 0: \epsilon'_i = \gamma P_i (I + \gamma P(I-\gamma P_{i})^{-1}(\pi_{i-1}-\pi_i))\epsilon_{i-1}$. Letting $E_j := (1-\beta)\sum_{i=0}^{j}\beta^{j-i}(\epsilon_{i}-\epsilon'_{i})$, $R_m := R_x + \gamma \tau \log |A|$ be an upper bound to $\ninf{\Tts\xi_0}$ and $\bR = \frac{R_m}{1-\gamma}$ be an upper bound to $\ninf{Q_\tau^0}$, we finally obtain
\begin{align}
    \Qts - \tau \xikp &\leq \Qts - \Tts \tau  \xik + \beta^{k} \Tts \xi_0 - E_k + (1-\beta)\beta^k Q_\tau^0,\\
    \Rightarrow \ninf{\Qts - \tau \xikp} &\leq \ninf{\Qts - \Tts \tau  \xik} + \beta^{k}\ninf{\Tts \xi_0} + \ninf{E_k} + (1-\beta)\beta^k \ninf{Q_\tau^0},\\
    &\leq \gamma \ninf{\Qts - \tau \xik} + \beta^{k}R_m + (1-\beta)\beta^k \bR+ \ninf{E_k},\\
    &\leq \gamma^{k+1} \ninf{\Qts} + (R_m + (1-\beta)\bR)\sum_{i=0}^{k}\gamma^i\beta^{k-i} + \sum_{i=0}^{k}\gamma^i\ninf{E_{k-i}},\\
    &= \gamma^{k+1} \ninf{\Qts} + \frac{2-\gamma -\beta}{1-\gamma}R_m\sum_{i=0}^{k}\gamma^i\beta^{k-i} + \sum_{i=0}^{k}\gamma^i\ninf{E_{k-i}}
\end{align}
\end{proof}
\subsection{Proof of Eq.~\eqref{eq:ximem2}}
\begin{proof}
    For $k = 0$,
    \begin{align}
        \xi_{1} &=\beta \times 0 + \alpha q_0 +  \frac{\alpha \beta^M}{1-\beta^M} (q_0 - 0),\\
        &= \alpha\left(1+\frac{\beta^M}{1-\beta^M}\right)q_0,\\
        &= \frac{\alpha}{1-\beta^M}q_0.
    \end{align}
If it is true for $k$, then
\begin{align}
    \xi_{k+1} &=\beta \frac{\alpha}{1-\beta^M}\sum_{i=0}^{M-1}\beta^iq_{k-1-i} + \alpha q_k +  \frac{\alpha \beta^M}{1-\beta^M} (q_k - q_{k-M}),\\
    &= \frac{\alpha}{1-\beta^M}\sum_{i=0}^{M-2}\beta^{i+1}q_{k-1-i} + \frac{\alpha \beta^M}{1-\beta^M} (q_{k-M} -  q_{k-M}) + \frac{\alpha}{1-\beta^M}q_k,\\
    &= \frac{\alpha}{1-\beta^M}\sum_{i=0}^{M-1}\beta^i q_{k-i}
\end{align}
\end{proof}

\subsection{Proof of Theorem \ref{thm:fmpi}}
As with policy iteration, we first need a policy improvement lemma
\begin{lem}[Approximate policy improvement of finite memory EPMD]
    For any $ k\geq 0$, $Q_{\tau}^{k+1}\geq Q_{\tau}^{k} - \gamma (I-\gamma P\pikp)^{-1}P [(\pikp - \pik) (\epsilon_k + \Delta_{k})]$, with $\Delta_{k}:=\frac{\beta^M}{1-\beta^M} (q_k - q_{k-M})$.
    \label{lem:fmpi:pi}
\end{lem}

\begin{proof}
We can see the policy $\pikp$ as the maximizer of Eq.~\eqref{eq:pup} if we would replace $q_k$ with $q_k + \frac{\beta^M}{1-\beta^M} (q_k - q_{k-M})$. From the non-negativity of the $\KL$ and the fact that $\KL(\pik;\pik) = 0$ we have
\begin{align}
\pik q_k - \tau h(\pik) & \leq \pikp q_k - \tau h(\pikp) + \frac{\beta^M}{1-\beta^M} (\pikp - \pik)(q_k - q_{k-M}),\\
\Rightarrow \pik \Qtk - \tau h(\pik) & \leq \pikp \Qtk - \tau h(\pikp) + \frac{\beta^M}{1-\beta^M} (\pikp - \pik)(q_k - q_{k-M}) + (\pikp - \pik)\epsilon_k.
\end{align}
Let $\Delta_{k} := \frac{\beta^M}{1-\beta^M} (q_k - q_{k-M})$. Writing down the (matrix) definition of $\Qtkp$ and $\Qtk$ gives
\begin{align}
    \Qtkp - \Qtk &= \gamma P  (\pikp\Qtkp - \tau h(\pikp)) - \gamma P (\pik\Qtk - \tau h(\pik)),\\
    &\overset{(i)}{\geq} \gamma P  (\pikp\Qtkp - \tau h(\pikp)) - \gamma P (\pikp \Qtk - \tau h(\pikp) + (\pikp - \pik) (\epsilon_k + \Delta_{k})),\\
    &= \gamma P \pikp (\Qtkp - \Qtk) - \gamma P [(\pikp - \pik) (\epsilon_k + \Delta_{k})],\\
        \Leftrightarrow \Qtkp - \Qtk &\overset{(ii)}{\geq} - \gamma (I-\gamma P\pikp)^{-1}P [(\pikp - \pik) (\epsilon_k + \Delta_{k})].
\end{align}
where in $(i)$ we have used the fact that $P$ is a probability matrix with only positive entries, and similarly in $(ii)$ for the matrix $(I-\gamma P \pikp)^{-1}=\sum_{i=0}^\infty (\gamma P\pikp)^i$.
\end{proof}

We are now ready to prove the main theorem
\begin{proof}

The beginning of the proof is similar to vanilla entropy regularized policy/value iteration
    \begin{align}
        \pikp q_k - \tau h(\pikp) &=  \pikp \left(\frac{1}{1-\beta}(\tau\xikp - \beta \tau \xik)\right) - \tau h(\pikp) - \pikp\Delta_{k},\\
        &= \pikp \left(\frac{1}{1-\beta}[\tau\xikp - \tau h(\pikp) - \beta (\tau \xik - \tau h(\pikp))]\right) - \pikp\Delta_{k},\\
        &\overset{(i)}{\geq} \frac{1}{1-\beta}[\pikp\tau\xikp - \tau h(\pikp) - \beta (\pik\tau \xik - \tau h(\pik))] - \pikp\Delta_{k}.
        \label{eq:fmpi:qkdk}
    \end{align}
    Using this inequality in $\qkp$ yields
    \begin{align}
        \qkp &= \Qtkp + \epsilon_{k+1},\\
        &= R + \gamma P(\pikp \Qtkp - \tau h(\pikp)) + \epsilon_{k+1},\\
        &\overset{(i)}{\geq} R + \gamma P(\pikp (q_k - \epsilon_k - \gamma \mu_{k+1} P (\pikp-\pik)(\epsilon_k+\Delta_{k})) - \tau h(\pikp)) + \epsilon_{k+1},\\
    \end{align}
    where for $(i)$ we used Lemma~\ref{lem:pi} and the definition of $q_k$. Using Eq.~\eqref{eq:fmpi:qkdk} on the following terms gives
    \begin{align}
        \begin{split}            
        R + \gamma P(\pikp q_k - \tau h(\pikp)) &\geq R + \gamma P\Big(\frac{1}{1-\beta}[\pikp\tau\xikp  - \tau h(\pikp) - \\
        &\quad\quad \beta (\pik \tau\xik - \tau h(\pik))]- \pikp\Delta_{k}\Big),
        \end{split}\\
        &= \frac{1}{1-\beta} (\Ttkp \tau \xikp - \beta \Ttk \tau \xik) - \gamma P\pikp\Delta_{k},\\
        &= \frac{1}{1-\beta} (\Tts \tau \xikp - \beta \Tts \tau \xik)  - \gamma P\pikp\Delta_{k}.
    \end{align}
    Completing with the rest of the terms finally gives
    \begin{align}
        \qkp &\geq  \frac{1}{1-\beta} (\Tts \tau \xikp - \beta \Tts \tau \xik)  - \gamma P \pikp (I+ \gamma \mu_{k+1}P(\pikp-\pik))(\epsilon_k + \Delta_{k}) + \epsilon_{k+1}.\\
    \end{align}
Let $A_{k+1} = \gamma P \pikp (I+ \gamma \mu_{k+1}P(\pikp-\pik))$ using this inequality in the definition of $\xikp$ gives
\begin{align}
    \tau \xikp &= \frac{1-\beta}{1-\beta^M} \sum_{i=0}^{M-1} \beta^{i} q_{k-i},\\
    &\geq \frac{1-\beta}{1-\beta^M} \sum_{i=0}^{M-1} \beta^{i} \left(\frac{1}{1-\beta}(\Tts \tau \xi_{k-i} - \beta \Tts \tau \xi_{k-i-1}) - A_{k-i}(\epsilon_{k-i-1}+\Delta_{k-i-1}) + \epsilon_{k-i}\right),\\
    &= \frac{1}{1-\beta^M}\Tts \tau \xik - \frac{\beta^{M}}{1-\beta^M} \Tts \xi_{k-M} + \frac{1-\beta}{1-\beta^M}\sum_{i=0}^{M-1}\beta^{i}(\epsilon_{k-i}- A_{k-i}(\epsilon_{k-i-1}+\Delta_{k-i-1})).
\end{align}
Letting $E_j :=  \frac{1-\beta}{1-\beta^M}\sum_{i=0}^{M-1}\beta^{i}(\epsilon_{k-i}- A_{k-i} \epsilon_{k-i-1})$, $\dqk = \frac{1-\beta}{1-\beta^M}\sum_{i=0}^{M-1}\beta^{i}A_{k-i}\Delta_{k-i-1}$, $R_m := R_x + \gamma \tau \log |A|$ be an upper bound to $\ninf{\Tts\xi_0}$ and $\bR = \frac{R_m}{1-\gamma}$ be an upper bound to $\ninf{Q_\tau^0}$, we finally obtain
\begin{align}
    \Qts - \tau \xikp &\leq \Qts - \frac{1}{1-\beta^M}\Tts \tau \xik + \frac{\beta^{M}}{1-\beta^M} \Tts \xi_{k-M} - E_k - \dqk,\\
    \Rightarrow \ninf{\Qts - \tau \xikp} &\leq \frac{\ninf{\Qts - \Tts \tau \xik} + \beta^M \ninf{\Qts - \Tts \tau \xi_{k-M}}}{1-\beta^M}+ \ninf{\dqk} + \ninf{E_k},\\
      &\leq \gamma\frac{\ninf{\Qts - \tau \xik} + \beta^M \ninf{\Qts - \tau \xi_{k-M}}}{1-\beta^M}+ \ninf{\dqk} + \ninf{E_k}
\end{align}
Let us now look into the term $\ninf{\dqk}$, and split it into policy evaluation error and distance to $\Qts$
\begin{align}
    \ninf{\dqk} &= \ninf{\frac{1-\beta}{1-\beta^M}\sum_{i=0}^{M-1}\beta^{i}A_{k-i}\Delta_{k-i-1}},\\
    &= \ninf{\frac{1-\beta}{1-\beta^M}\sum_{i=0}^{M-1}\beta^{i}A_{k-i}\frac{\beta^M}{1-\beta^M} (q_{k-i-1} - q_{k-i-1-M})},\\
    &= \ninf{\frac{\beta^M(1-\beta)}{(1-\beta^M)^2}\sum_{i=0}^{M-1}\beta^{i}A_{k-i}(Q_{k-i-1} - Q_{k-i-1-M}+\epsilon_{k-i-1} - \epsilon_{k-i-1-M})},\\
    \begin{split}
        &\leq \ninf{\frac{\beta^M(1-\beta)}{(1-\beta^M)^2}\sum_{i=0}^{M-1}\beta^{i}A_{k-i}(\epsilon_{k-i-1} - \epsilon_{k-i-1-M})} \\
        &\quad+ \ninf{\frac{\beta^M(1-\beta)}{(1-\beta^M)^2}\sum_{i=0}^{M-1}\beta^{i}A_{k-i}(Q_{k-i-1} - Q_{k-i-1-M})}.
    \end{split}\label{eq:fmpi:split}
\end{align}
    To bound the infinite norm of $A_k$, we note that $(1-\gamma)\mu_k$ is a probability matrix (the state distribution induced by policy $\pik$). Using the sub-additivity of norms and the fact that the multiplication of probability matrices is a probability matrix with infinite norm equal to 1, we have
    \begin{align}
        \ninf{A_k} &= \ninf{\gamma P \pikp \left(I+ \frac{\gamma}{1-\gamma}((1-\gamma)\mu_{k+1})P(\pikp-\pik)\right)},\\
        &\leq \gamma + \frac{2\gamma^2}{1-\gamma}.
    \end{align}
    Looking now at the rightmost inner sum in Eq.~\eqref{eq:fmpi:split} gives
    \begin{align}
        &\ninf{\sum_{i=0}^{M-1}\beta^{i}A_{k-i}(Q_{k-i-1} - Q_{k-i-1-M})} \\
        &\leq {\sum_{i=0}^{M-1}\beta^{i}\ninf{A_{k-i}(Q_{k-i-1} - Q_{k-i-1-M})}} ,\\
        &\leq \left(\gamma + \frac{2\gamma^2}{1-\gamma}\right)\sum_{i=0}^{M-1}\beta^{i}\ninf{Q_{k-i-1} - Q_{k-i-1-M}},\\
        &\leq \left(\gamma + \frac{2\gamma^2}{1-\gamma}\right)\sum_{i=0}^{M-1}\beta^{i}\left(\ninf{\Qts - Q_{k-i-1}} + \ninf{\Qts - Q_{k-i-1-M}}\right),\\
        &\overset{(i)}{\leq} \frac{2}{1-\gamma}\left(\gamma + \frac{2\gamma^2}{1-\gamma}\right)\sum_{i=0}^{M-1}\beta^{i}\left(\ninf{\Qts - \tau\xi_{k-i-1}} + \ninf{\Qts - \tau\xi_{k-i-1-M}}\right),
    \end{align}
    where $(i)$ is due to Lemma~\ref{lem:xitoq}.
    Let $z_k = \ninf{\Qts - \tau\xi_k}$ and $\err_k$ grouping all the error terms
    \begin{align}
        T_k = \ninf{E_k} + \frac{\beta^M(1-\beta)}{(1-\beta^M)^2}\ninf{\sum_{i=0}^{M-1}\beta^{i}A_{k-i}(\epsilon_{k-i-1} - \epsilon_{k-i-1-M})}.
    \end{align}    
    Putting everything together we have
    \begin{align}
        z_{k+1} &\leq \gamma\frac{z_k + \beta^M z_{k-M}}{1-\beta^M} + \frac{\beta^M(1-\beta)}{(1-\beta^M)^2} \frac{2(\gamma+\gamma^2)}{(1-\gamma)^2} \sum_{i=0}^{M-1}\beta^{i}\left(z_{k-i-1} + z_{k-i-1-M}\right) + \err_k.\label{eq:fmpi:zk}
    \end{align} 
    Let us first study the sequence $\{z_k\}$ without policy evaluation errors and try to upper bound it with a simpler sequence. We define the sequence $\{x_k\}$ for all integers $k$ by 
    \begin{align}
        x_k = \ninf{\Qts},\ \text{for all } k \leq 0,
    \end{align}
    and for $k \geq 0$ we let
    \begin{align}
        x_{k+1} = \gamma\frac{x_k + \beta^M x_{k-M}}{1-\beta^M} + \frac{\beta^M(1-\beta)}{(1-\beta^M)^2} \frac{2(\gamma+\gamma^2)}{(1-\gamma)^2} \sum_{i=0}^{M-1}\beta^{i}\left(x_{k-i-1} + x_{k-i-1-M}\right).
    \end{align}
    We first find a condition for which the sequence is strictly decreasing starting from $k \geq 0$. For $k = 0$ we have that
    \begin{align}
        x_1 &= \left(\gamma \frac{1 + \beta^M}{1-\beta^M} + \frac{\beta^M}{1-\beta^M} \frac{4(\gamma+\gamma^2)}{(1-\gamma)^2}\right)x_{0}.
    \end{align}
    This will be strictly decreasing if 
    \begin{align}
        \gamma \frac{1 + \beta^M}{1-\beta^M} + \frac{\beta^M}{1-\beta^M} \frac{4(\gamma+\gamma^2)}{(1-\gamma)^2} &< 1,\\
        \Leftrightarrow \gamma ({1 + \beta^M}) + \beta^M \frac{4(\gamma+\gamma^2)}{(1-\gamma)^2}&< (1-\beta^M),\\
        \Leftrightarrow \left(1 + \gamma + \frac{4(\gamma+\gamma^2)}{(1-\gamma)^2}\right)\beta^M &< 1 - \gamma,\\
        \Leftrightarrow  \log\left(1 + \gamma + \frac{4(\gamma+\gamma^2)}{(1-\gamma)^2}\right) + M\log\beta &< \log(1 - \gamma),\\
        \Leftrightarrow M &> \log \frac{(1-\gamma)^3}{(1+\gamma)(1-\gamma)^2 + 4(\gamma+\gamma^2)}/\log{\beta}.\label{eq:fmpi:mcond}
    \end{align}
    As it is true for $k = 0$ and the sequence is constant for $k < 0$, assume now that the sequence is strictly decreasing from there on, up to some positive index $k$. Then since all the weights of past terms are positive, we can replace all terms by their predecessors and we immediately have that
    \begin{align}
        x_{k+1} &< \gamma\frac{x_{k-1} + \beta^M x_{k-M-1}}{1-\beta^M} + \frac{\beta^M(1-\beta)}{(1-\beta^M)^2} \frac{2(\gamma+\gamma^2)}{(1-\gamma)^2} \sum_{i=0}^{M-1}\beta^{i}\left(x_{k-i-2} + x_{k-i-2-M}\right),\\
        &= x_k.
    \end{align}
    Thus the sequence $\{x_k\}$ is non-increasing for all $k$ if $M$ satisfies the inequality in \eqref{eq:fmpi:mcond}. For such values of $M$ we will now find an upper bounding sequence that has a simpler geometric form.  Letting $c = \frac{\beta^M}{1-\beta^M} \left(\frac{4(\gamma+\gamma^2)}{(1-\gamma)^2} + \gamma\right)$, and since the sequence is non-decreasing, we have for all $k \geq 0$ that     
    \begin{align}
        x_{k+1} &\leq \frac{\gamma}{1-\beta^M} x_k + c x_{k-2M}.\label{eq:fmpi:simprec}
    \end{align}
    Let us now try to find a rate $d\in(0,1)$ such that for all $k$ we have
    \begin{align}
        x_{k} \leq d^{k} x_{0}.
    \end{align}
    For all $k\leq 0$, the above inequality holds for any $d\in(0,1)$.
    Now, if the upper bounding is true up to some index $k$ then using Eq.~\eqref{eq:fmpi:simprec} we have
    \begin{align}
        x_{k+1} &\leq \left(\frac{\gamma}{1-\beta^M} d^{k} + c d^{k-2M}\right)x_0,\\
        &= d^k \left(\frac{\gamma}{1-\beta^M} + c d^{-2M}\right)x_0.
    \end{align}
    The smallest acceptable $d$ would be one such that 
    \begin{align}
        \frac{\gamma}{1-\beta^M} + c d^{-2M} &= d,\\
        \Leftrightarrow d^{2M+1} - \frac{\gamma}{1-\beta^M} d^{2M} - c&= 0.
    \end{align}
    Let $f(d) = d^{2M+1} - \frac{\gamma}{1-\beta^M} d^{2M} - c$, we have that $f(\frac{\gamma}{1-\beta^M}) = -c < 0$ and that $f(1) = 1 - \frac{\gamma}{1-\beta^M} - c > 0$ from the above condition on $M$. Since $f$ is continuous and increasing between these two values it accepts a unique root $d_0\in (\frac{\gamma}{1-\beta^M}, 1)$ which would satisfy the sought geometric sequence upper bound for all $k$. 

    We now turn to the part of the sequence of $z_k$ that depends on the error terms $T_k$. Define for $k\leq 0$
    \begin{align}
        y_k = 0,
    \end{align}
    and for $k \geq 0$
    \begin{align}
        y_{k+1} &= \gamma\frac{y_k + \beta^M y_{k-M}}{1-\beta^M} + \frac{\beta^M(1-\beta)}{(1-\beta^M)^2} \frac{2(\gamma+\gamma^2)}{(1-\gamma)^2} \sum_{i=0}^{M-1}\beta^{i}\left(y_{k-i-1} + y_{k-i-1-M}\right) + \err_k.\label{eq:fmpi:yk}
    \end{align}
    Let $\bar{T} = \underset{0\leq i\leq k}{\max} T_i$, and $\gamma_M = \frac{\gamma}{1-\beta^M}$. Then we have that
    \begin{align}
        y_k \leq \frac{\bar{T}}{1-(\gamma_M+c)}.
    \end{align}
    Indeed, it is true for $k\leq 0$, and if it is true up to $k$ then
    \begin{align}
        y_{k+1} &\leq (\gamma_M + c)  \frac{\bar{T}}{1-(\gamma_M+c)} + \err_k,\\
        &\leq (\gamma_M + c)  \frac{\bar{T}}{1-(\gamma_M+c)} + \bar{T},\\
        &= \frac{\bar{T}}{1-(\gamma_M+c)}.
    \end{align}
    Replacing in Eq.~\eqref{eq:fmpi:yk}, we have
    \begin{align}
        y_{k+1} &\leq \gamma_M y_k + c  \frac{\bar{T}}{1-(\gamma_M+c)} + T_k,\\
                &\leq \gamma_M^{k+1}y_0 + \sum_{i=0}^{k} \gamma_M^{i} \left(T_{k-i} + c \frac{\bar{T}}{1-\gamma_M-c}\right),\\
                &= \sum_{i=0}^{k} \gamma_M^{i} \left(T_{k-i} + c \frac{\bar{T}}{1-\gamma_M-c}\right).
    \end{align}
    Finally, because the upper bound $z_{k+1}$ in Eq.~\eqref{eq:fmpi:zk} has linear dependencies on previous $z_i$ terms ($i\leq k$), we immediately have that $z_{k} \leq x_{k} + y_{k}$. Indeed, it is true for $k=0$, since $z_0 = \ninf{\Qts} = x_0 + y_0$. And if we assume that it is true for $k$, using Eq.~\eqref{eq:fmpi:zk}, we immediately have that it is true for $k+1$. Thus
    \begin{align}
        z_{k+1} &\leq x_{k+1} + y_{k+1},\\
        &\leq d_0^{k+1} \ninf{\Qts} + \sum_{i=0}^{k} \gamma_M^{i} \left(T_{k-i} + c \frac{\bar{T}}{1-\gamma_M-c}\right).
    \end{align}
\end{proof}

\clearpage

\section{Additional Experiments}
\label{app:experimental_results}
\begin{figure}[h]
    \centering
    \hfill\includegraphics[width=0.95\textwidth]{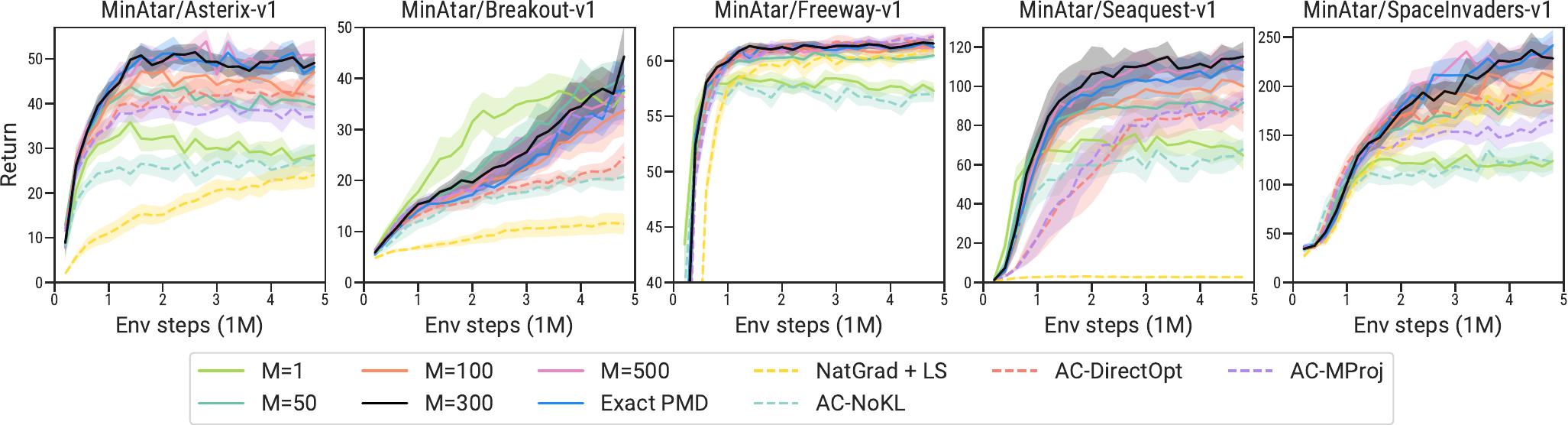}\par
    \vspace{1em}
    \includegraphics[width=\textwidth]{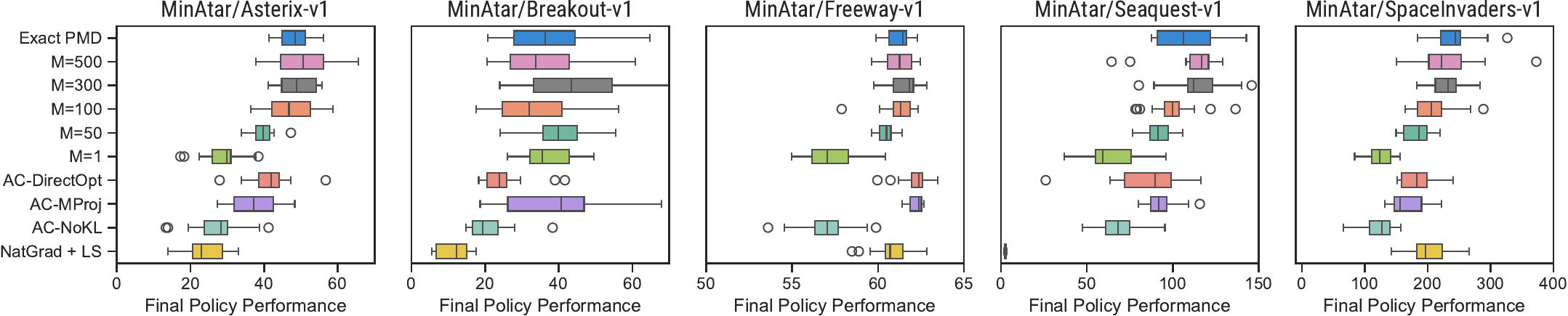}
    \caption{{Evaluation of StaQ for different memory sizes $M$, as well as AC PMD baselines, on MinAtar environments. Learning curves (\textbf{top}) show mean and 95\% confidence interval for 20 seeds, and box plots (\textbf{bottom}) of final policy performance.}}
    \label{fig:extraM}
\end{figure}
{\subsection{Learning curves and final performance box plots}
\label{app:Mablation}

{Figure~\ref{fig:extraM} shows the mean and 95\% confidence interval of the performance of StaQ for different choices of $M$.
Setting $M=1$ corresponds to no KL-regularization as discussed in App.~\ref{app:notsac} and can be seen as an adaptation of SAC to discrete action spaces. $M=1$ can be unstable and lead to performance collapse. Adding KL-regularization and increasing $M$ helps to improve performance, but in general we note diminishing returns in increasing $M$, with $M=100$ matching the performance of exact EPMD in most environments and $M=300$ being virtually indistinguishable from exact EPMD. Compared to natural policy gradient and other AC methods, performance is generally improved on most tasks, making StaQ a potentially better alternative to other PMD implementations, at least on environments where Q-function computations can be efficiently batched.

\clearpage

\subsection{Comparison with deep RL baselines}
\label{app:full_results}
We summarize all performance comparisons with the deep RL baselines in Fig.~\ref{fig:xp_returns} and Table~\ref{tab:final_perf}. 
We provide a discussion of the MountainCar environment and some of the challenges of exploration in an entropy-regularized setting in App.~\ref{app:explore}.

\begin{figure}[h!]
    \centering
    \includegraphics[width=0.97\linewidth]{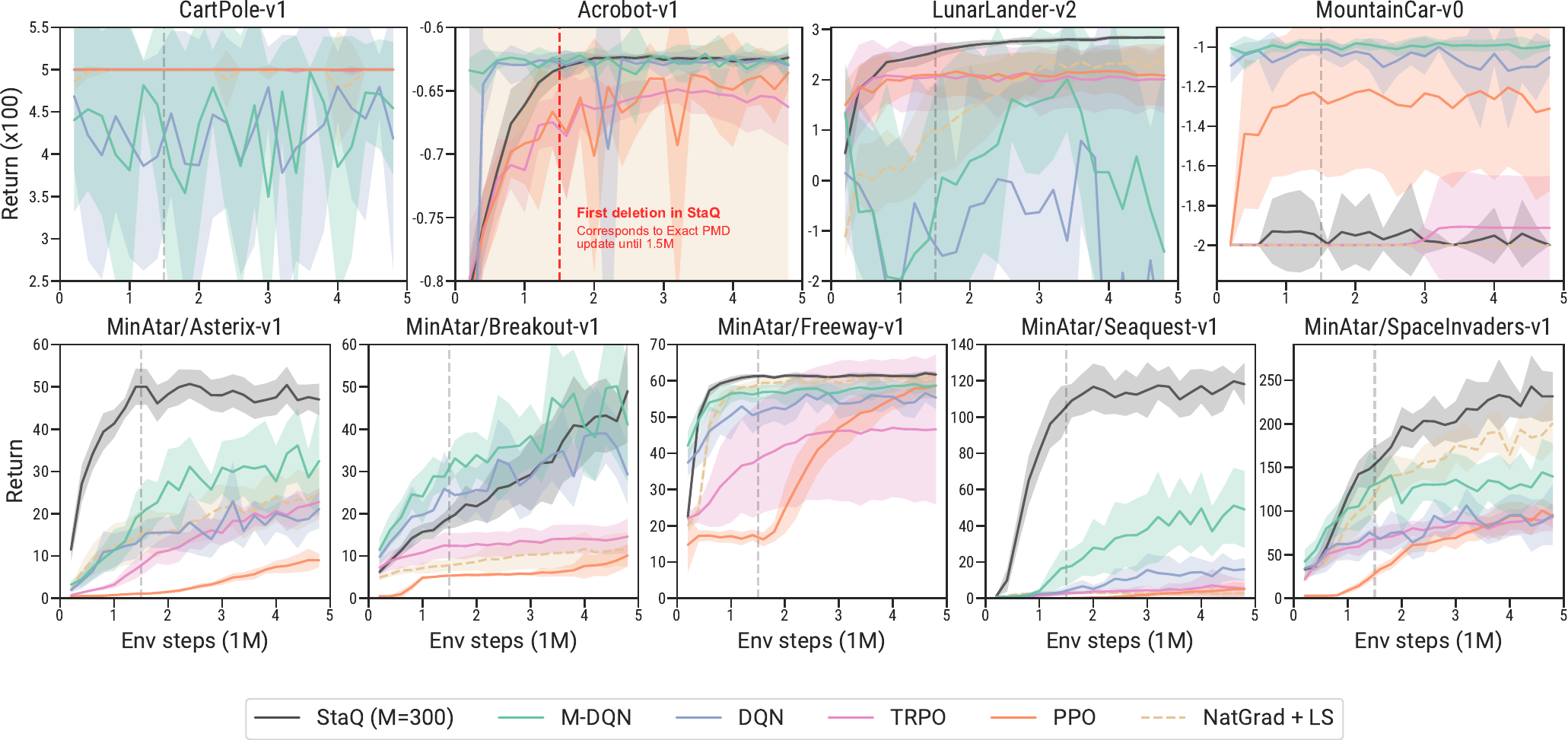}
    \caption{Policy performance of StaQ ($M=300$) vs deep RL baselines across all environments. Results showing the mean and standard deviation across 10 seeds. Note that for $M = 300$ the first deletion occurs at 1.5M timesteps -- \textbf{this deletion has no visible impact on policy performance}. %
    }
    \label{fig:xp_returns}
\end{figure}

\begin{table}[h!]
\centering
\begin{tabular}{lrrrrrr}
\toprule
                          & StaQ (M=300)   & DQN          & M-DQN        & PPO          & TRPO         &   NatGrad + LS \\
\midrule
 CartPole-v1              & \textbf{500}   & 418          & 454          & \textbf{500} & \textbf{500} &            497 \\
 Acrobot-v1               & \textbf{-63}   & \textbf{-63} & \textbf{-63} & -64          & -66          &            -97 \\
 LunarLander-v2           & \textbf{283}   & -343         & -142         & 209          & 201          &            240 \\
 MountainCar-v0           & -199           & -105         & \textbf{-99} & -131         & -191         &           -200 \\
\midrule
 MinAtar/Asterix-v1       & \textbf{49}    & 21           & 32           & 9            & 23           &             24 \\
 MinAtar/Breakout-v1      & \textbf{44}    & 29           & 41           & 10           & 15           &             11 \\
 MinAtar/Freeway-v1       & \textbf{62}    & 55           & 59           & 59           & 47           &             61 \\
 MinAtar/Seaquest-v1      & \textbf{115}   & 16           & 49           & 5            & 5            &              3 \\
 MinAtar/SpaceInvaders-v1 & \textbf{229}   & 96           & 140          & 94           & 93           &            203 \\
\bottomrule
\end{tabular}
\caption{Final performance on all environments comparing StaQ ($M=300$) and Deep RL baselines, as well as \texttt{NatGrad + LS} for comparison with TRPO.}
\label{tab:final_perf}
\end{table}

\clearpage

\subsection{Entropy regularization does not solve exploration}

\begin{figure}
    \centering
    \def\plotwidth{0.32}
    \includegraphics[width=\plotwidth\linewidth]{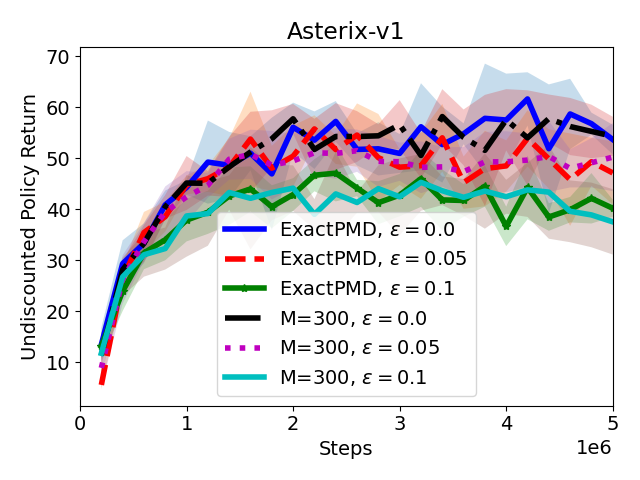}
    \includegraphics[width=\plotwidth\linewidth]{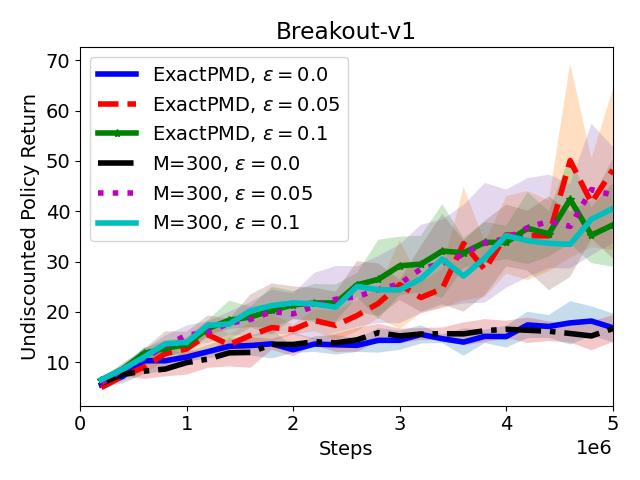}
    \includegraphics[width=\plotwidth\linewidth]{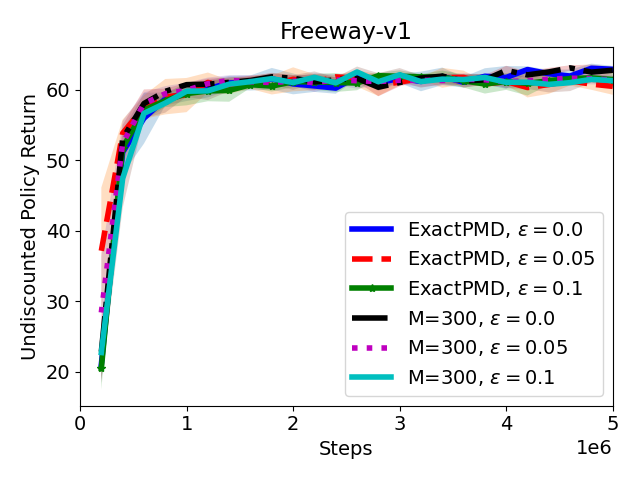}
    \includegraphics[width=\plotwidth\linewidth]{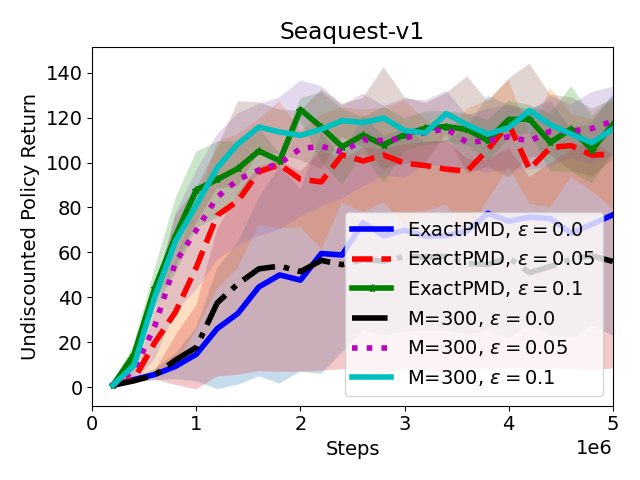}
    \includegraphics[width=\plotwidth\linewidth]{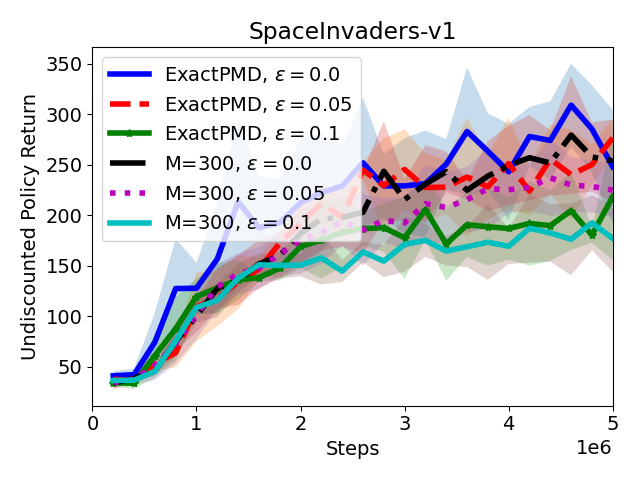}
    \caption{Impact of the $\epsilon$ exploration parameter of StaQ on performance in MinAtar tasks. Plots showing mean and std. deviation over 5 seeds.}
    \label{fig:enter-label}
\end{figure}

The behavioural policy $\pi_k^b$ used for exploration is based on sampling from the softmax policy with logits $\xi_k$. However, we observe that some environments benefit from the addition of a constant probability $\epsilon=0.05$ of sampling a uniformly-random action throughout the learning phase. This observation that uniform sampling can give rise to better exploration that softmax policies has been made previously, especially in sparse-reward settings such as Atari. See e.g. the discussion in App.~B2 in the M-DQN paper~\cite{MDQN}. In this appendix we ablate the choice of this $\epsilon$-softmax policy and discuss modified behavioural policies.

In a set of additional experiments we evaluate the impact $\epsilon$ has on performance. To this end, we launch StaQ with $M=300$ and Exact PMD on MinAtar. We see that $\epsilon$ exploration is important for both Breakout and Seaquest, but has less of an impact on other tasks. In fact, on Asterix and SpaceInvaders, increasing the epsilon decreases the performance slightly, perhaps because of the increase to the off-policiness of the data which might deteriorate policy evaluation. As such, our choice of $\epsilon=0.05$ in the main experiment appears to provide a good trade-off between improving exploration and limiting the off-policiness of data on MinAtar tasks. In all cases, independent of the choice of $\epsilon$, StaQ with $M=300$ remains close to Exact PMD in performance.

\label{app:explore}
\begin{figure}[t!]
    \centering
    \includegraphics[width=0.35\linewidth]{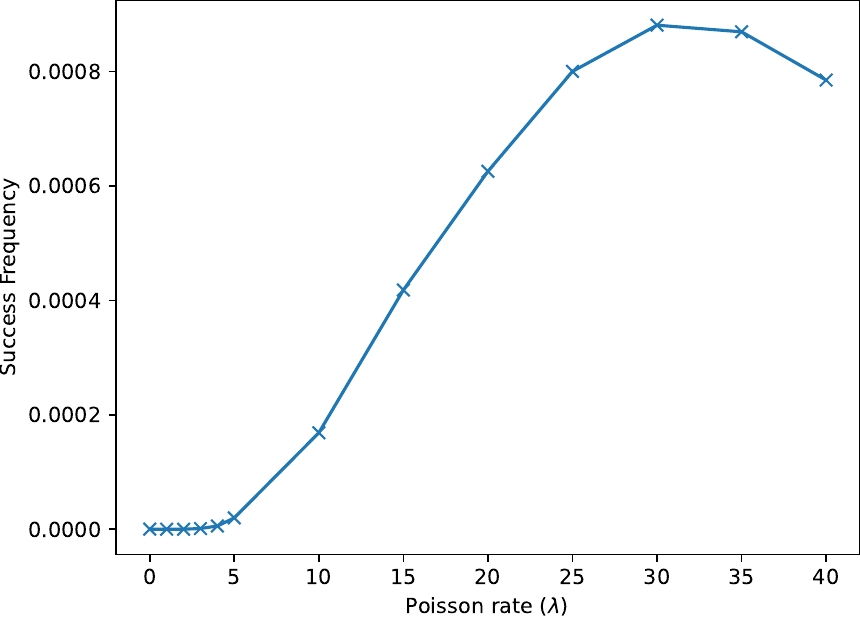}
    \includegraphics[width=0.25\linewidth]{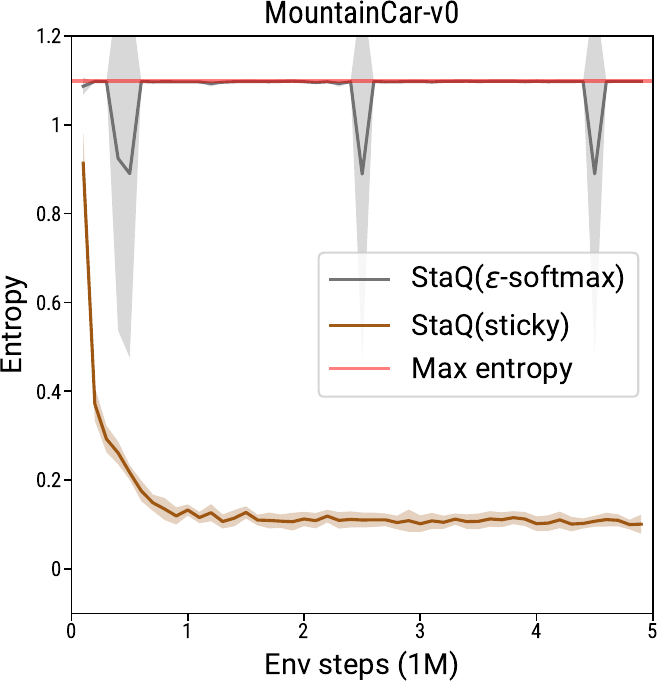}
    \includegraphics[width=0.36\linewidth]{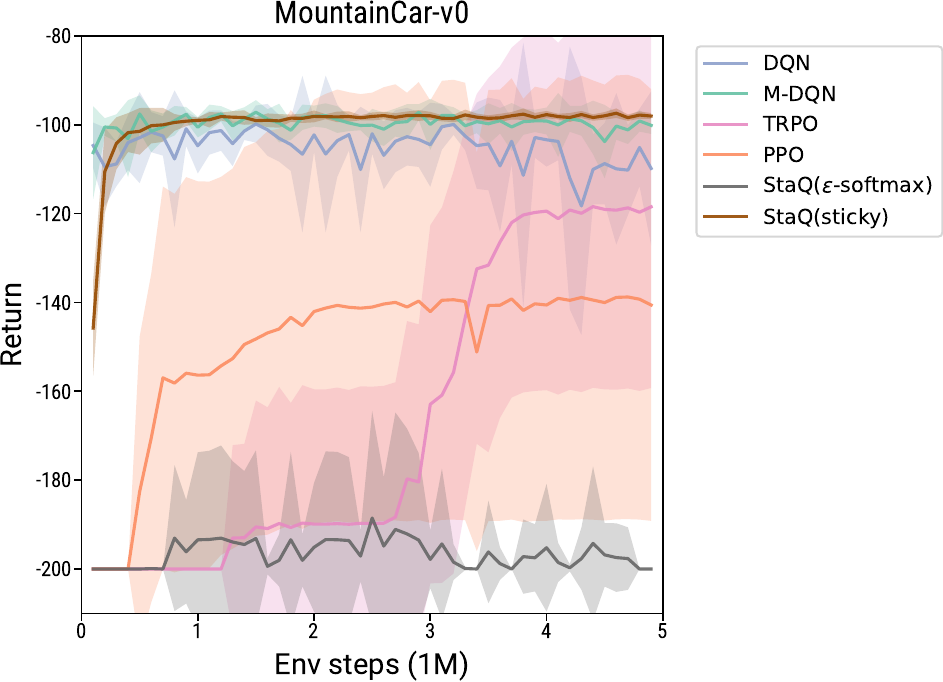}
    \caption{\textbf{Left:} Frequency of non-zero rewards of a uniform policy with sticky actions for different choice of Poisson rate $\lambda$ on MountainCar over 5M timesteps. \textbf{Middle:} Entropy of learned policies under different behavior policies. Entropy of the uniform (Max entropy) policy plotted for reference. \textbf{Right:} Policy returns for StaQ with different behavior policies and deep RL baselines on MountainCar. Adding sticky actions to StaQ's behavior policy fixes its performance on this task.}
    \label{fig:xp_explore}
\end{figure}

StaQ and exact EPMD achieve strong performance on all 9 environments except on MountainCar where they fail to learn. In this section, we perform additional experiments to understand the failure of StaQ on MountainCar. In short, it appears that the initial uniform policy---which has maximum entropy---acts as a strong (local) attractor for this task: StaQ starts close to the uniform policy, and exploration with this policy does not generate a reward signal in MountainCar. As StaQ does not observe a reward signal in early training, it quickly converges to the uniform policy which has maximum entropy, but also has extremely small probability of generating a reward signal. Indeed, we unrolled a pure uniform policy on MountainCar for 5M steps, and never observed a reward.

However, StaQ is not limited to a specific choice of behavior policy, and choosing a policy that introduces more correlation between adjacent actions, like a simple ``sticky'' policy allows StaQ to solve MountainCar. This policy samples an action from $\pik$ and applies it for a few consecutive steps, where a number of steps is drawn randomly from Poisson($\lambda$) distribution (in our experiments with StaQ we fix the rate of Poisson distribution at $\lambda = 10$). In Fig.~\ref{fig:xp_explore}, we can see that StaQ with the same hyperparameters for classical environments (see Table~\ref{tab:hyper_staq}) and a "sticky" behavior policy %
manages to find a good policy for MountainCar matching the best baseline. The final policy demonstrates much lower entropy compared to the default policy that fails at learning for this environment. We focused on this paper on the benefits of entropy regularization in averaging evaluation error. While it is believed that entropy might help with exploration, these observations are a good reminder that entropy regularization remains a heuristic exploration strategy that does not replace a theoretically grounded strategy, which is beyond the scope of this paper that only focuses on using entropy regularization to reduce the error floor of approximate policy iteration when using function approximators.

\subsection{Averaging of evaluation error} 
\label{sec:app:shinrl}
\begin{figure}
    \centering
    \includegraphics[width=0.4\linewidth]{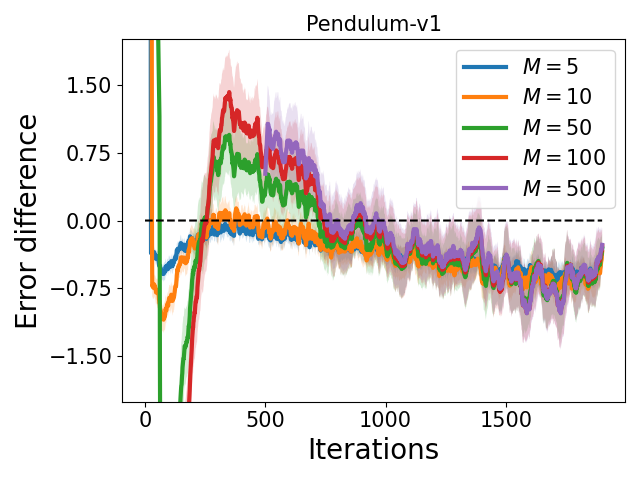}
    \includegraphics[width=0.4\linewidth]{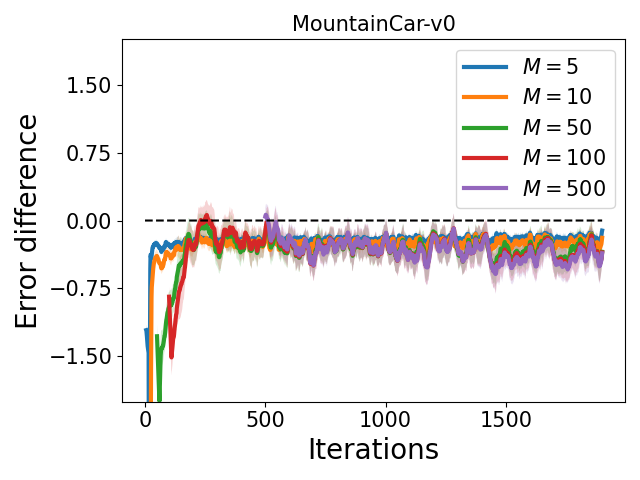}
    \caption{Policy evaluation difference $\ninf{\frac{1-\beta}{1-\beta^M}\sum_{i=0}^{M-1}\beta^{i}\epsilon_{k-i}} - \ninf{\epsilon_k}$ between the weighted average of the last $M$ evaluation errors versus the last evaluation error $\epsilon_k$, for different values of $M$, for a PMD value iteration scheme. As the number of iterations increases, we observe an averaging of policy evaluation errors effect due to the $\KL$ regularization even on simple classic control tasks. Results averaged over 50 seeds, showing mean and 95\% confidence interval. For readability, the results are smoothed with a moving average of window size 20.}
    \label{fig:cancelerr}
\end{figure}
We explore in this section whether the averaging effect of the evaluation errors due to the $\KL$ regularization does occur in practice. While we would want the per iteration evaluation errors $\{\epsilon_i\}_{i=1}^k$ to behave like uncorrelated random noise, in practice in a deep RL setting, the neural network at the current iteration is often reused to estimate the Q-function at the next iteration. Datasets also often overlap in replay buffers for the estimation of successive Q-function, creating correlations between successive Q-functions and their successive errors. Taking all this into account, it becomes unclear whether a cancellation of evaluation errors does indeed occur in practice. To evaluate this effect, we make use of the ShinRL library \citep{shinrl} that allows the computation of the exact Q-function for small scale environments such as Pendulum and MountainCar. The library operates by internally building a discrete MDP, and associating to each discrete state an auxiliary vectorial observation that can be processed by a function approximator. 

Using these environments, we implement a PMD value iteration algorithm following exactly the definition of Sec.~\ref{sec:epmdvi}, using a neural network $q_k$ for approximating the current soft Q-function $\Qtk$, but using exact policy update on the discrete states. The learning of $q_k$ follows a standard procedure where a target is computed using $q_{k-1}$ and regressed for $5000$ gradient steps. The policy is then updated using $\eta=1$ and $\tau=0.03$. At each iteration, we compute the exact Q-function $\Qtk$ and the evaluation error $\epsilon_k$ using the internally discrete MDP. 

To evaluate the potential cancellation of errors, we then compute the difference between $\ninf{\frac{1-\beta}{1-\beta^M}\sum_{i=0}^{M-1}\beta^{i}\epsilon_{k-i}} - \ninf{\epsilon_k}$, i.e. between the weighted average of the last $M$ evaluation errors versus the last evaluation error $\epsilon_k$, for different values of $M$. Fig.~\ref{fig:cancelerr} shows that as the number of iterations increases, the averaging of policy evaluation error does drop below the last error $\epsilon_k$, even for these simple environments. 

\section{Connection to Actor-Critic and PMD baselines}
\label{app:connection_baselines}

In this section, we describe how the baselines SAC~\citep{SAC_app}, TRPO~\citep{TRPO}, MDPO~\citep{tomar2022mirror} and ECPO~\citep{ECPO} connect to StaQ, and more generally to the Entropy-regularised Policy Mirror Descent Objective~\eqref{eq:pup}.

\subsection{Soft Actor Critic}
\label{app:notsac}

By setting $M=1$ in StaQ, we remove the KL-divergence regularization and only keep the entropy bonus. This baseline can also be seen as an adaptation of SAC to discrete action spaces: indeed, if we set $M=1$ in Eq.~(\ref{eq:intro:staq}) we recover the policy logits
\begin{align}
    \xi_{k+1} &= \frac{\alpha}{1-\beta^M}\sum_{i=0}^{M-1}\beta^i Q_{\tau}^{k-i}\\
    &= \frac{\alpha}{1-\beta}Q_{\tau}^{k} \\
    &= \frac{Q_{\tau}^{k}}{\tau},
\end{align}
where the last line is due to $\alpha \tau = 1-\beta$. This results in a policy of the form $\pi_{k+1} \propto \exp\left(\frac{Q_{\tau}^k}{\tau}\right)$. Meanwhile, for SAC, the actor network is obtained by minimizing the following problem (Eq. 14 in \citet{SAC_app}) for states sampled from some replay buffer $\cal D$
\begin{align}
\pi_{k+1} = \underset{\pi}{\arg\min}\ \E_{s\sim {\cal D}}\left[\text{KL}\left(\pi(s) \middle| \frac{\exp\left(\frac{Q_{\tau}^k(s)}{\tau}\right)}{Z_{\text{norm.}}}\right)\right]
    \label{eq:sacobj}
\end{align}
However, in the discrete action setting, we can sample directly from $\exp\left(\frac{Q_{\tau}^k}{\tau}\right)$---which is the minimizer of the above KL-divergence term---and we do not need an explicit actor network. As such StaQ with $M=1$ could be seen as an adaptation of SAC to discrete action spaces. Nevertheless, we also consider as a baseline an actor-critic version that updates the actor by attempting to solve Eq.~\ref{eq:sacobj} via gradient descent, resulting in an even closer baseline to SAC. We label this baseline \texttt{AC-NoKL (SAC)}.

\subsection{Approximate Policy Mirror Descent}
\label{sec:app:acpmd}

In the presence of function approximation, simple updates in logit space of a Boltzmann policy do not in general translate to simple updates in the Q-function parameter space. Let $\Theta_Q$ and $\Theta_\pi$ respectively be the parameter spaces of Q-functions and policy logits; these parameter spaces can for instance be subsets of $\R^d$ for some integer $d$. Let $\xi_\theta:S\times A\mapsto \R$, with $\theta\in\Theta_\pi$ be a function that provides the logits of a policy $\pi_\theta\propto\exp(\xi_\theta)$ and let $q_{\theta'}:S\times A\mapsto \R$, with $\theta'\in\Theta_\pi$, be the (approximate) Q-function associated to $\pi_\theta$. We want to find $\xi_{\theta''}$ as the solution to the entropy regularized policy update in Eq.~\ref{eq:pup}. We know that ideally we would have $\xi_{\theta''} = \beta \xi_{\theta} + \alpha q_{\theta'}$, but this does not give an expression for $\theta''$ since in general the policy maximizing Eq.~\ref{eq:pup} is not necessarily parameterized by $\beta {\theta} + \alpha {\theta'}$. 

One exception to the above claim is when the Q-function and the logits function are linear w.r.t. some predefined feature function, in which case $\xi_{\beta {\theta} + \alpha {\theta'}} = \beta \xi_{\theta} + \alpha q_{\theta'}$. This is the so-called compatible feature setting of policy gradient \citep{pgthm,gestnat,copos} where the Q-function and the logits share the same linear-in-feature function approximation class and in which case, policy gradient and natural policy gradient become equivalent~\citep{peterscomp}. However, beyond the linear-in-feature case, the closed form solution of the entropy regularized policy update in the space of logits does not yield a trivial update in parameter space.

One approach to policy update in parameter space is to solve the following optimization problem 
\begin{align}
    \underset{\theta''\in\Theta_\pi}{\arg\max}\ \theta''\cdot\nabla_{\theta}J_\tau(\pi_\theta) - \eta\E_{s\sim \mu_{\pi_\theta}} \KL(\pi_{\theta''}(s); \pi_\theta(s)), \label{eq:prelim:mdavg}
\end{align}
where $J_\tau(\pi_\theta)$ is the policy return $\pi_\theta$, i.e. the expectation of the $V^{\pi_\theta}$ over states sampled from a predefined initial state distribution. Interestingly, \cite{Cen} showed that in the tabular case (i.e. when optimizing directly over the logit space), the above problem is equivalent to the maximization in Eq.~\ref{eq:pup} over each state independently. In the general case however, because of the approximation in the optimization of Eq.~\ref{eq:prelim:mdavg} or because of the use of a restricted policy class, we are likely to obtain a new policy $\pi_{\theta''}$ that is worse than $\pikp$ in the sense of the policy update objective defined in Eq.~\ref{eq:pup}.

\subsection{Approximate PMD Algorithms}

We now describe several approximate algorithms for solving either the original PMD objective~\eqref{eq:pup} or a state-averaged formulation thereof.

\textbf{NatGrad + LS (TRPO).} We implement \texttt{NatGrad~+~LS~(TRPO)} that updates the policy by approximately solving Eq.~\ref{eq:prelim:mdavg} using TRPO's conjugate gradient approach. We use a post-update line search step that constrains the $\KL$ between successive policies to be under an $\epsilon_\text{KL}$ threshold. This baseline has subtle differences with TRPO: for instance even when an entropy bonus is used, TRPO does not learn soft value functions and only regularizes the policy update with an additional entropy term. We found that by matching StaQ's setting, ``NatGrad + LS'' is comparable in performance to \texttt{stable-baselines3}'s TRPO on most of the tasks but can outperform it on others, especially on MinAtar tasks, as seen from the results in App.~\ref{app:full_results}.

\textbf{AC-MProj (ECPO).} The entropy regularized PMD solution is $\pi_{k+1}\propto\pi_k^\beta \exp\left({\alpha\Qtk}\right)$. Tracking this recursive policy definition exactly results in an infinite sum, but we can instead approximate $\pi_k^\beta \exp\left({\alpha\Qtk}\right)$ with an actor network. Specifically, we consider a second baseline inspired by ECPO \citep{ECPO} that advocates for minimizing the M-Projection between the actor and the PMD solution
\begin{align}
    \pi_{k+1} = \underset{\pi}{\arg\min}\ \E_{s\sim {\cal D}}\left[\text{KL}\left(\frac{\pi_k^\beta \exp\left({\alpha\Qtk}\right)}{Z_{\text{norm.}}} \middle| \pi(s) \right)\right],
    \label{eq:ecpoobj}
\end{align}
where states are again sampled from a replay buffer $\cal D$. This differs from the forward-KL objective in e.g. SAC; compare with~\eqref{eq:sacobj}. The authors of ECPO argue that the M-Projection will better preserve the support of the distribution and prevent the premature elimination of actions during exploration. We will label this baseline \texttt{AC-MProj (ECPO)}.

\textbf{AC-DirectOpt (MDPO)} Finally, we consider an actor-critic baseline inspired by MDPO \citep{tomar2022mirror} that optimizes a state averaged version of the entropy regularized PMD update, i.e. 
\begin{align}
    \pi_{k+1} = \underset{\pi}{\arg\max}\ \E_{s\sim {\cal D}} \left[\Qtk(s) \cdot \pi(s) - \tau h(\pi(s)) -\eta \KL(\pi(s);\pik(s))\right],
\end{align}
by performing a few gradient steps over the above objective. Unlike ECPO, MDPO skips the closed form solution and optimizes directly the entropy and $\KL$ regularized objective by gradient ascent. We will call this baseline \texttt{AC-DirectOpt (MDPO)}.

\textbf{Experimental comparisons with Exact PMD.} As can be seen from the main experimental results in Fig.~\ref{fig:mainxp_m} (and also the further results in App.~\ref{app:experimental_results}), AC-DirectOpt and AC-MProj being approximations of Exact PMD, end-up having worse performance on several tasks. Comparing the AC methods between themselves, AC-DirectOpt and AC-MProj have close performance on the aggregate metrics, while AC-NoKL lags behind the other two and shows again the benefits of $\KL$ regularization on policy update. Compared to StaQ, the performance of the $\KL$ regularized AC methods (i.e. AC-DirectOpt and AC-MProj) is between that of StaQ with $M=1$ and $M=50$, while AC-NoKL performs worse than $M=1$. The improved performance of StaQ with higher $M$ comes at little performance cost as can be seen in the table in Fig~\ref{fig:overview}, where the runtime remains close to that of $M=1$, whereas all three AC methods have to perform double the amount of gradient steps to update the actor in addition to the critic, resulting in a 35 to 45\% longer run time than StaQ with $M=1$.

\section{Hyperparameters and Implementation Details}
\label{app:hyperparams}

\textbf{Policy Evaluation.} In all our experiments, we use an ensemble of two neural networks, similarly to e.g. SAC \citep{SAC_app}, to evaluate a $Q$-function and therefore two stacked neural networks for $\xi$-logits. In particular, we optimize the current Q-function weights $\theta$ to minimize the loss $\mathcal{L}(\theta)$,
\begin{align}
    \mathcal{L}(\theta) &= \mathbb{E}_{\left(s, a\right) \sim \mathcal{D}}\left[\frac{1}{2}\left(Q_\theta\left(s, a\right)-\hat{Q}\left(s, a\right)\right)^2\right] \label{eq:fqi_loss}\\
    \hat{Q}(s,a) &:= R(s,a) + \gamma \mathbb{E}_{s'\sim\mathcal{D},a\sim\pi(s')}\left[ \frac{1}{2} \sum_{i =1}^2 Q_{\hat{\theta}_i}( s',a') - \tau h(\pi (s'))\right]
\end{align}
computing the mean over the target Q-functions with weights $\hat{\theta}_1, \hat{\theta}_2$.

\textbf{Hyperparameters.} Below, we provide the full list of hyperparameters used in our experiments\footnote{Code will be released to an open-source repository upon publication.}. StaQ's hyperparameters are listed in Table~\ref{tab:hyper_staq}, while the hyperparameters for our baselines are provided in Tables~\ref{tab:hyper_ppo}-\ref{tab:hyper_mdqn}. For TRPO and PPO, we use the implementation provided in \texttt{stable-baselines4}~\citep{stable-baselines3}, while we used our in-house PyTorch implementation of (M)-DQN\footnote{We will add the link to the in-house library upon publication.}.

To account for the different scales of the reward between environments, we apply a different reward scaling to the Classic environments and MinAtar. Note that this is equivalent to inverse-scaling the entropy weight $\tau$ and KL weight $\eta$, ensuring that $\xik$ is of the same order of magnitude for all environments. To account for the varying action dimension $|A|$ of the environments, we set the \emph{scaled entropy coefficient} $\bar{\tau}$ as a hyperparameter, defined by $\bar{\tau} = \tau \log |A|$, rather than directly setting $\tau$. Furthermore, the entropy weight is linearly annealed from its minimum and maximum values.

\begin{table}[h!]
\centering
\begin{tabular}{cccc}
\toprule
Hyperparameter              & Classic         & MinAtar         \\ \midrule
Discount ($\gamma$)         & 0.99            & 0.99            \\
Memory size ($M$)           & 300             & 300             \\ \midrule
Policy update interval      & 5000            & 5000            \\
Target type                 & hard            & hard            \\
Target update interval      & 200             & 200             \\
Epsilon                     & 0.05            & 0.05             \\ 
Reward scale                & \textbf{10}     & \textbf{100}             \\
KL weight ($\eta$)          & 20     & 20     \\
Initial scaled ent. weight  & 2.0             & 2.0             \\
Final scaled ent. weight    & 0.4             & 0.4             \\
Ent. weight decay steps     & \textbf{500K}   & \textbf{1M}     \\ \midrule
Architecture                & $\bf 256 \times 2$  & \textbf{Conv(16, 3, 3) + 128 MLP}  \\
Activation function         & ReLU            & ReLU            \\
Learning rate               & 0.0001          & 0.0001 \\
Optimizer                   & Adam            & Adam \\
Replay capacity             & 50K             & 50K             \\
Batch size                  & 256             & 256             \\ \bottomrule
\end{tabular}
\caption{StaQ hyperparameters, with parameters which vary across environment types in bold.}
\label{tab:hyper_staq}
\end{table}

\begin{table}[h!]
\centering
\begin{tabular}{cccc}
\toprule
Hyperparameter             & Classic            & MinAtar                            \\ \midrule
Discount factor ($\gamma$) & 0.99               & 0.99                               \\
Horizon                    & 2048               & 1024                               \\
Num. epochs                & 10                 & 3                               \\
Learning starts            & 5000               & 20000                              \\ \midrule
GAE parameter              & 0.95               & 0.95                               \\
VF coefficient             & 0.5                & 1                                  \\
Entropy coefficient        & 0                  & 0.01                               \\ 
Clipping parameter         & 0.2                & $0.1 \times \alpha$ \\ \midrule
Optimizer                  & Adam               & Adam        \\
Architecture               & $256 \times 2$      & Conv(16, 3, 3) + 128 MLP                     \\
Activation function        & Tanh               & Tanh                               \\
Learning rate              & $3 \times 10^{-4}$ & $2.5 \times 10^{-4} \times \alpha$ \\
Batch size                 & 64                 & 256                                 \\ \bottomrule
\end{tabular}
\caption{PPO hyperparameters, based on~\citep{PPO}. In the MinAtar environments $\alpha$ is linearly annealed from 1 to 0 over the course of learning.}
\label{tab:hyper_ppo}
\end{table}

\begin{table}[]
\centering
\begin{tabular}{cccc}
\toprule
Hyperparameter             & Classic            & MinAtar            \\ \midrule
Discount factor ($\gamma$) & 0.99               & 0.99               \\
Horizon                    & 2048               & 2048  \\
Learning starts            & 5000               & 20000 \\ \midrule
GAE parameter              & 0.95               & 0.95               \\
Stepsize                   & 0.01               & 0.01               \\ \midrule
Optimizer                  & Adam               & Adam        \\
Architecture               & $256 \times 2$      & Conv(16, 3, 3) + 128 MLP     \\
Activation function        & Tanh               & Tanh               \\
Learning rate              & $3 \times 10^{-4}$ & $2.5 \times 10^{-4}$ \\
Batch size                 & 64                 & 256                 \\ \bottomrule
\end{tabular}
\caption{TRPO hyperparameters, based on~\citep{TRPO}.}
\label{tab:hyper_trpo}
\end{table}

\begin{table}[]
\centering
\begin{tabular}{cccc}
\toprule
Hyperparameter             & Classic             & MinAtar            \\ \midrule
Discount factor ($\gamma$) & 0.99                & 0.99               \\
Target update interval     & 100                 & 8000               \\
Update-to-data ratio       & 1                   & 1               \\
Epsilon                    & 0.1                 & 0.1               \\
Decay steps                & 20K                 & 20K               \\ \midrule
M-DQN temperature           & 0.03               & 0.03               \\
M-DQN scaling term          & 1.0                & 0.9                \\
M-DQN clipping value        & -1                 & -1                 \\ \midrule
Architecture               & $512 \times 2$      & Conv(16, 3, 3) + 128 MLP     \\
Activation function        & ReLU                & ReLU               \\
Learning rate              & $1 \times 10^{-3}$  & $2.5 \times 10^{-4}$ \\
Optimizer                  & Adam                & Adam \\
Replay capacity            & 50K                 & 1M                 \\
Batch size                 & 128                 & 32                 \\ \bottomrule
\end{tabular}
\caption{MDQN and DQN hyperparameters, based on~\citep{MDQN, RevRainbow}}
\label{tab:hyper_mdqn}
\end{table}

\clearpage

\subsection{Hyperparameter ablations for Natural Gradient and Actor-Critic baselines}
\label{app:klnatgrad}
\begin{figure}[t]
    \centering
    \includegraphics[width=0.24\linewidth]{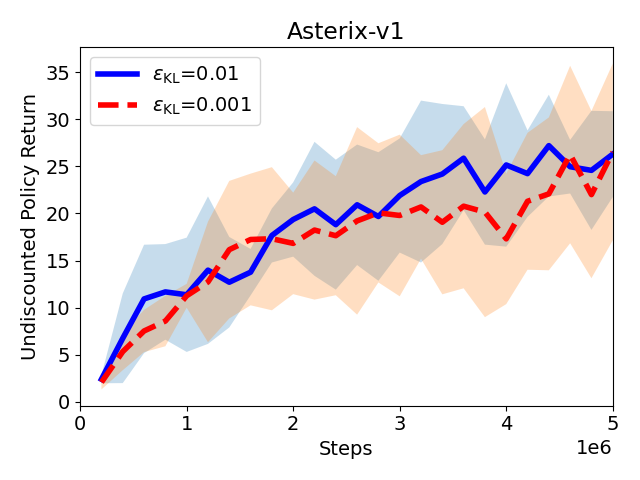}
    \includegraphics[width=0.24\linewidth]{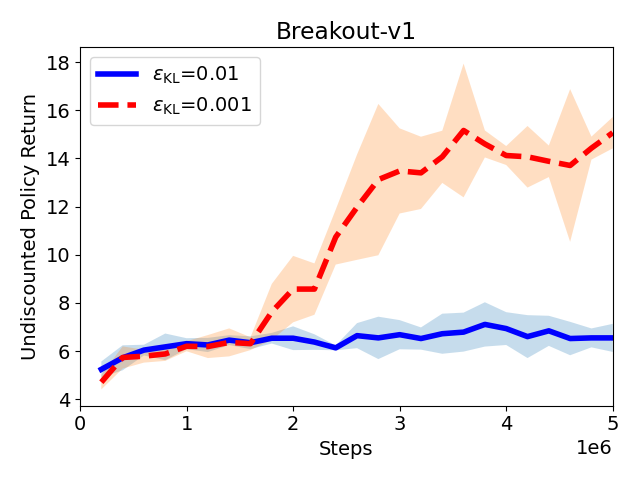}
    \includegraphics[width=0.24\linewidth]{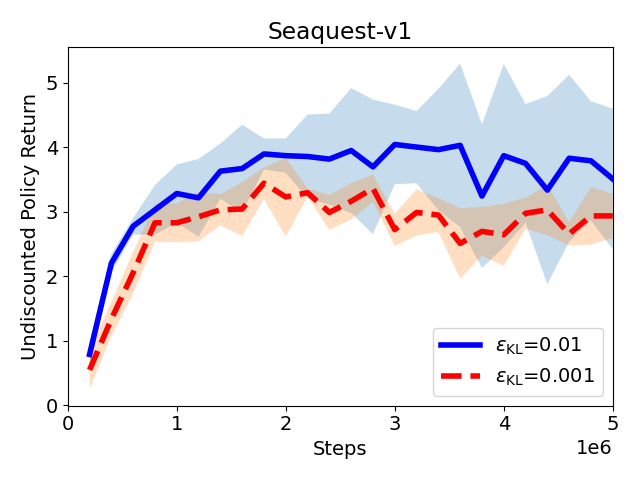}
    \includegraphics[width=0.24\linewidth]{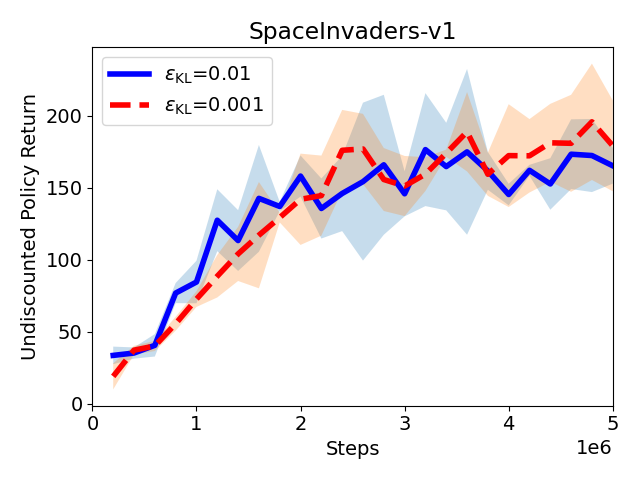}
    \caption{Mean and standard deviation of policy performance of ``NatGrad + LS (TRPO)'' baseline for two values of $\epsilon_\text{KL}$ averaged over 3 seeds.}
    \label{fig:natgradls}
\end{figure}
\begin{figure}[t]
    \centering
    \includegraphics[width=0.24\linewidth]{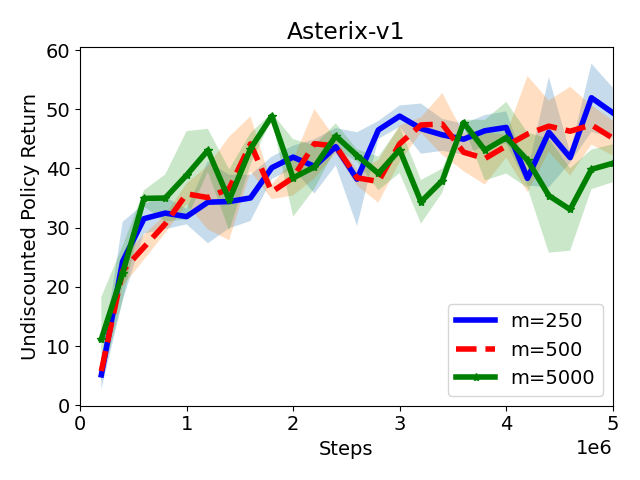}
    \includegraphics[width=0.24\linewidth]{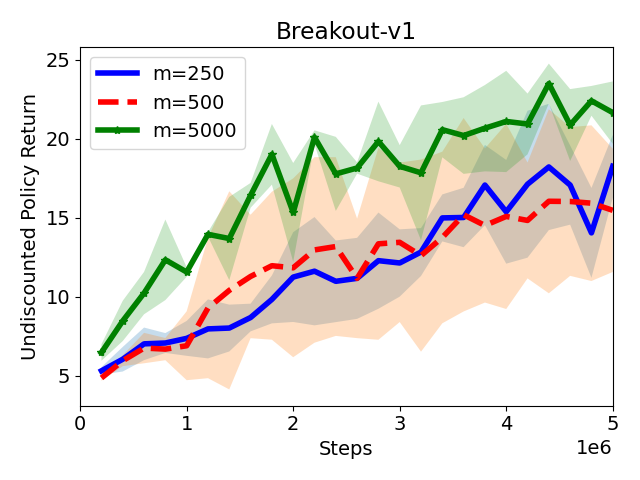}
    \includegraphics[width=0.24\linewidth]{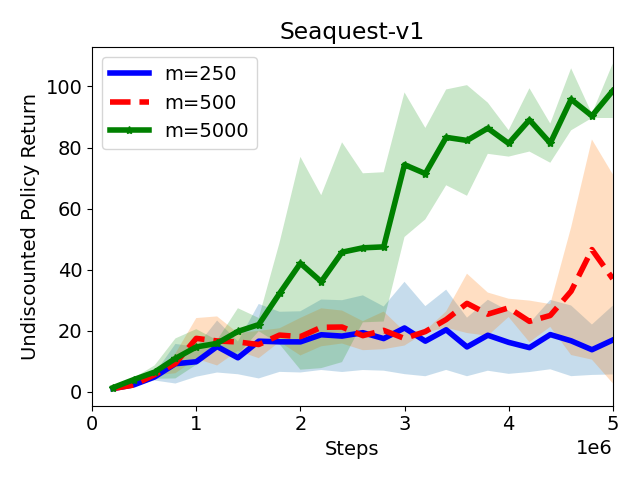}
    \includegraphics[width=0.24\linewidth]{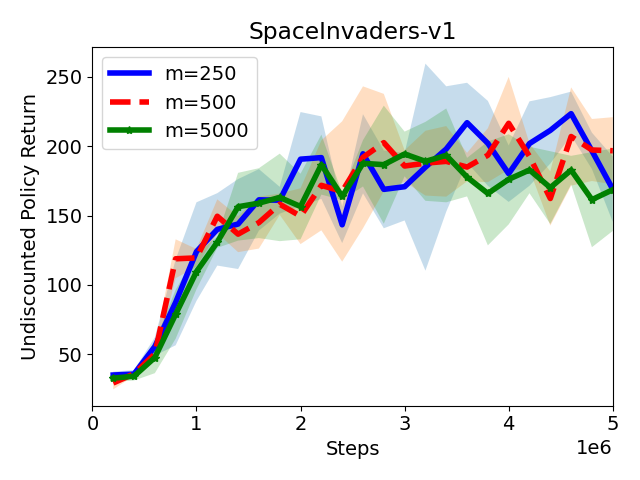}
    \includegraphics[width=0.24\linewidth]{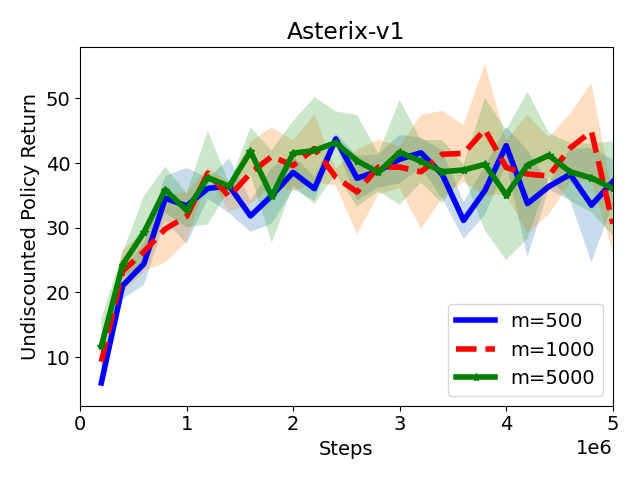}
    \includegraphics[width=0.24\linewidth]{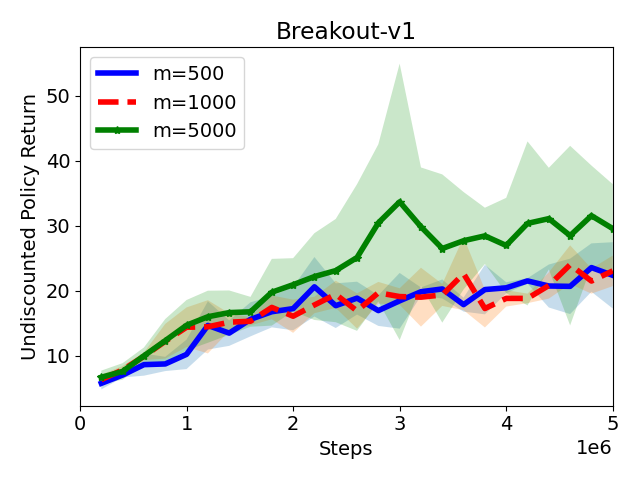}
    \includegraphics[width=0.24\linewidth]{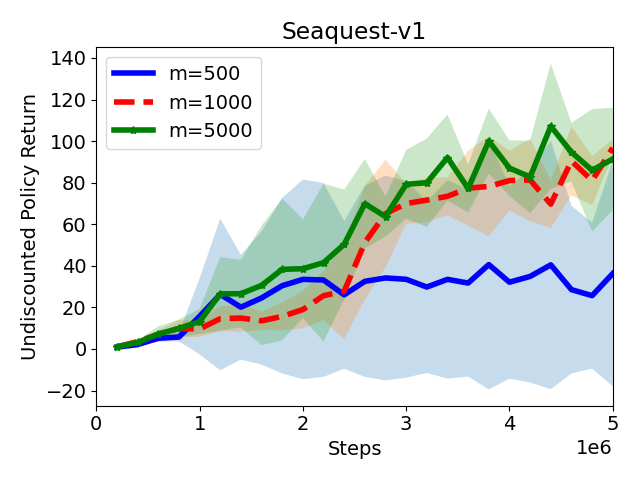}
    \includegraphics[width=0.24\linewidth]{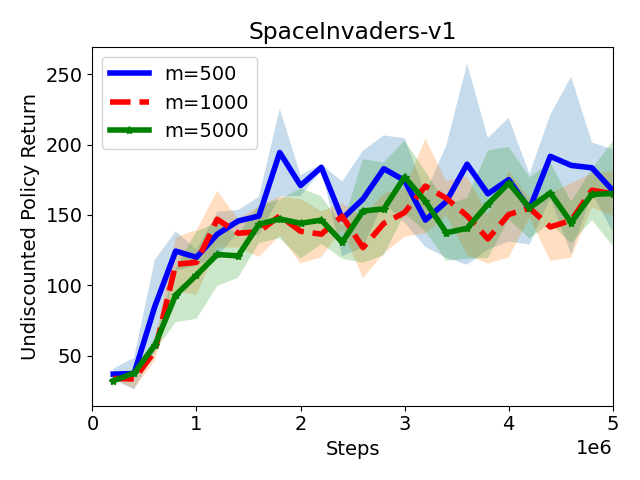}
    \caption{Sensitivity of AC-DirectOpt (MDPO) (\textbf{top}) and AC-MProj (ECPO) (\textbf{bottom}) to the number of gradient steps performed to update the actors. Plots show mean and std. deviation over 3 seeds. See App.~\ref{sec:app:acpmd} for details.}
    \label{fig:MDPOabl}
\end{figure}

For a fair comparison, the NatGrad + LS and Actor-Critic baselines use the same hyperparameters as StaQ, except for the algorithm-specific parameters $\epsilon_\text{KL}$ and $m$, for which we provide sensitivity studies here.

\textbf{NatGrad + LS.} The baseline ``NatGrad + LS'' introduces an additional hyperparameter $\epsilon_\text{KL}$ compared to the standard PMD setting that StaQ follows. To properly set this hyperparameter we perform a preliminary study using 3 seeds on a subset of the MinAtar tasks. The results in Figure~\ref{fig:natgradls} show that on a few tasks, the value of the $\epsilon_\text{KL}$ does not impact greatly the performance, but when there are significant differences between the two algorithms, $\epsilon_\text{KL}=0.001$ is often better than the value $\epsilon_\text{KL}=0.01$ used in may of TRPO's implementations, e.g. in \texttt{stable-baselines3}~\citep{stable-baselines3}. As such, we use $\epsilon_\text{KL}=0.001$ for all experiments involving  ``NatGrad + LS''.

\textbf{MDPO, ECPO and SAC.} An important parameter for MDPO and ECPO is $m$, the number of gradient steps that are used to update the actor. We perform a 3 seed hyperparameter sensitivity analysis w.r.t. $m$ on 4 MinAtar tasks. Results show in Fig.~\ref{fig:MDPOabl} that on some tasks, this parameter is actually not very crucial, but on some other tasks, using more gradient steps for the actor seems to improve performance. As such, for the remained of this section, we will use $m=5000$ for both  AC-DirectOpt (MDPO) and AC-MProj (ECPO), which corresponds to performing one actor gradient update per time-step in the environment. This is the same setting as in SAC, and we will use the same number of actor updates for the AC-NoKL (SAC) baseline. 

\clearpage

\section{Extension to continuous action spaces}
\label{app:cont_actions}
\begin{figure}[t]
    \centering
    \def\plotwidth{0.4}
    \includegraphics[width=\plotwidth\linewidth]{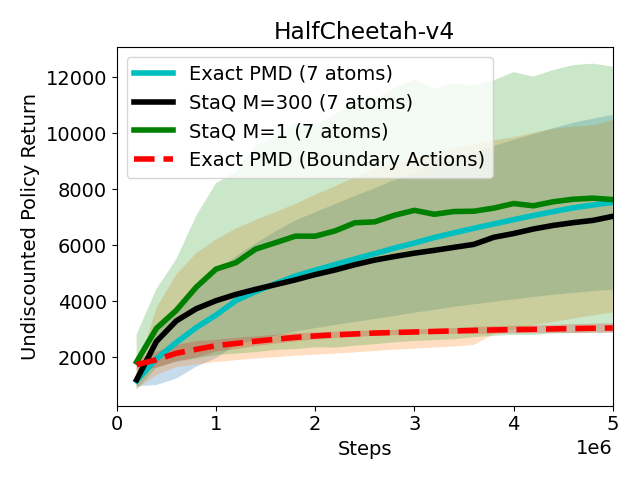}
    \includegraphics[width=\plotwidth\linewidth]{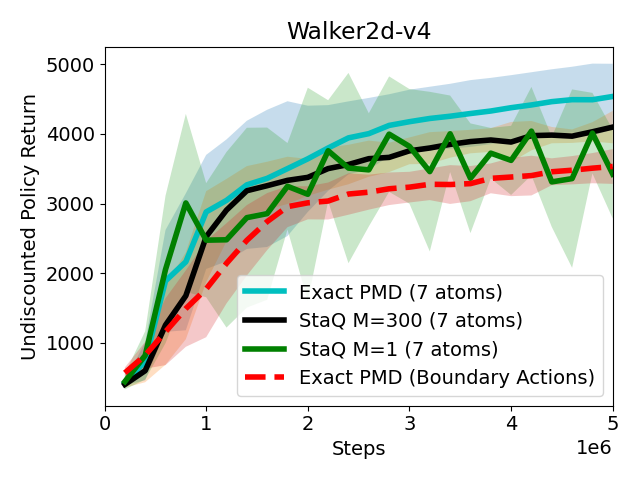}
    \includegraphics[width=\plotwidth\linewidth]{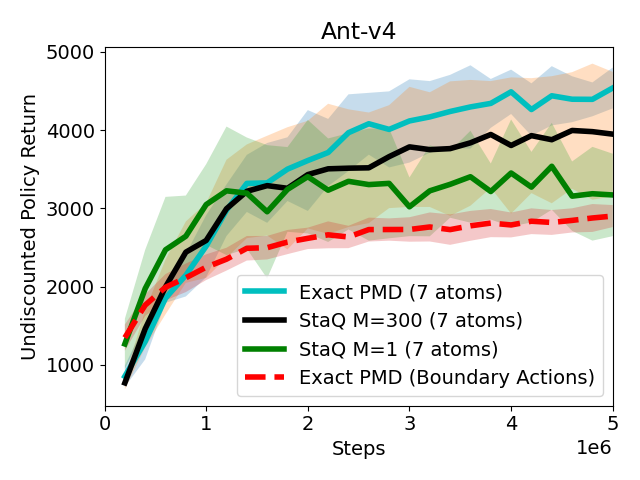}
    \includegraphics[width=\plotwidth\linewidth]{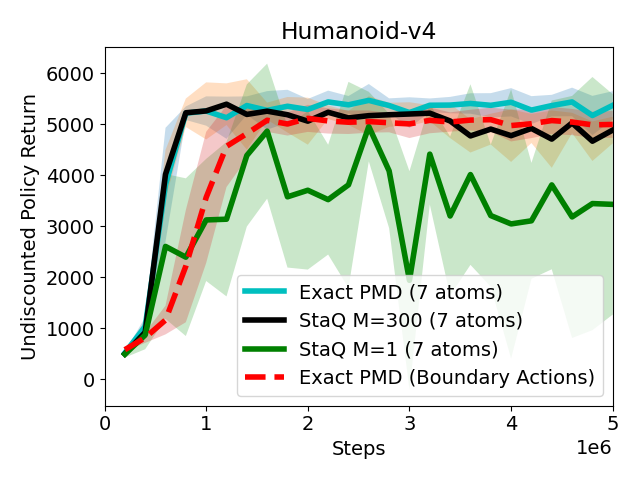}
    \caption{StaQ on MuJoCo with action space discretized by just considering the boundary actions, versus using the independence assumption from~\cite{Tang} with 7 atoms. Results averaged over 5 seeds.}
    \label{fig:fact}
\end{figure}
While StaQ is inherently a discrete-actions algorithm, in this appendix we describe some promising ways in which we could extend it to continuous action spaces. One interesting research questions in this direction is to investigate the improvements a Boltzmann distribution could bring to exploration compared to the unimodal Gaussian distribution that is commonly used in continuous action deep RL algorithms.

The simplest approach is to discretise the continuous action space by only considering the ``boundary actions'', similarly to~\citep{bangbang}. To illustrate, the discrete version of a Mujoco task with action space $A = [-1,1]^d$ consists in several $2d$ dimensional vectors that have zeroes everywhere except at entry $i\in\{1,\dots,d\}$ that can either take a value of $1$ or $-1$; to that we add the zero action, for a total of $2d+1$ actions. 

Another direction to tackle continuous action spaces is to leverage an independence assumption in the policy following~\cite{Tang}. Concretely, \citeauthor{Tang} discretize each action dimension into a finite number of atoms and assume that each action dimension follows an independent distribution---which is also assumed by continuous action RL algorithms using diagonal Gaussian policies. In our setting, this translates to the logit space, which is also the space of Q-functions, by assuming separability of the Q-function approximation i.e. that $q_k$, the approximation of the Q-function $Q_k$, is given by
\begin{equation}
    q_k(s,a) = \sum_{i=1}^{\text{dim(A)}} q_k^{[i]}(s, a^{[i]}),\label{eq:fact}
\end{equation}
where $a^{[i]}$ is entry $i$ of the action vector $a$ and $\{q_k^{[i]}\}_{i=1}^{\text{dim}(A)}$ is a set of real functions. This independence assumption makes it possible to consider fine grained discretizations with extremely large but finite action spaces, from which one can sample exactly and efficiently by sampling each action dimension independently. Incidentally, this separability assumption could allow a more efficient implementation of the Gumbel-max trick discussed above.

To illustrate how this could benefit StaQ, we have implemented the discretization from \cite{Tang} in combination with StaQ using $K=7$ atoms by action dimension. Concretely, the Q-network outputs K values for each action dimension, and the Q-function for a given action is computed using Eq.~\ref{eq:fact}. We have run 5 seeds of StaQ using this Q-network architecture on the 4 MuJoCo environments shown in Fig.~\ref{fig:fact}, comparing it to the simple Boundary Actions discretisation described previously. All hyperparameters of StaQ are the same except for the target update rate of the Q-target (1000 for HalfCheetah and Walker, 2500 for Ant and Humanoid). From the plot, we can see that the Boundary Actions discretisation works rather well for most environments, and the finer-grained discretization brings further improvements on all but the Humanoid-v4 task. Moreover, using StaQ improves over previous results reported by \cite{Tang}, that used PPO and TRPO as base learners on most MuJoCo tasks, showing the impact our work could have in this line of research.

\section{Pseudocode of StaQ}

We provide in this section the pseudocode of StaQ in Alg.~\ref{alg:staq}. As an approximate policy iteration algorithm, StaQ comprises three main steps: i) data collection, ii) policy evaluation iii) policy improvement. Data collection (Line 4-5) consist in interacting with the environment to collect transitions of type (state, action, reward, next state) that are stored in a replay buffer. A policy evaluation algorithm (Eq.~\ref{eq:fqi_loss}) is then called to evaluate the current Q-function $\Qtk$ using the replay buffer. Finally, the policy update is optimization-free and simply consists in stacking the Q-function as discussed in Sec.~\ref{sec:stack}. After $K$ iterations, the last policy is returned.
\begin{algorithm}
   \caption{StaQ (Finite-memory entropy regularized policy mirror descent)}
   \label{alg:staq}
\begin{algorithmic}[1]
   \STATE {\bfseries Input:} An MDP $\cal M$, a memory-size $M$, Number of samples per iteration $N$, Replay buffer size $D$, Initial behavior policy $\pi_0^b$, entropy weight $\tau$, $\KL$ weight $\eta$, $\epsilon$-softmax exploration parameter 
   \STATE {\bfseries Output:} Policy $\pi_{K}\propto \exp(\xi_{K})$
   \FOR{$k=0$ {\bfseries to} $K - 1$}
   \STATE Interact with $\cal M$ using the behavior policy $\pi_k^b$ for $N$ times steps \label{lalg:int}
   \STATE Update replay buffer ${\cal D}_k$ to contain the last $D$ transitions \label{lalg:buff}
   \STATE Learn $\Qtk$ from ${\cal D}_k$ using a policy evaluation algorithm (Eq.~\ref{eq:fqi_loss})
   \STATE Obtain logits $\xikp$ by stacking the last $M$ Q-functions (see Sec.~\ref{sec:stack}) following the finite-memory EPMD update of Eq.~\ref{eq:ximem2}.
   \STATE Set $\pikp\propto\exp(\xikp)$ and $\pi_{k+1}^b$ to an $\epsilon$-softmax policy over $\pikp$
   \ENDFOR
\end{algorithmic}
\end{algorithm}

\end{document}